\documentclass{article}

\usepackage{microtype}
\usepackage{graphicx}
\usepackage{booktabs}
\usepackage{hyperref}

\usepackage{amsmath,amssymb,bm}
\usepackage[normalem]{ulem}
\usepackage[dvipsnames]{xcolor}
\usepackage{graphicx}
\usepackage{bbm}
\usepackage{mathtools}
\usepackage{enumitem}
\usepackage{mathtools}
\usepackage{amsthm}
\usepackage{multirow}
\usepackage{subfig}
\usepackage{makecell}
\usepackage[outline]{contour}
\usepackage{thmtools, thm-restate}

\usepackage[accsupp]{axessibility}

\DeclareMathOperator*{\argmax}{arg\,max}
\DeclareMathOperator*{\argmin}{arg\,min}

\theoremstyle{plain}
\newtheorem{theorem}{Theorem}[section]

\newtheorem{lemma}[theorem]{Lemma}

\theoremstyle{definition}

\theoremstyle{remark}

\DeclarePairedDelimiter\norm{\lVert}{\rVert}

\renewcommand{\vec}[1]{\bm{#1}}

\newcommand{\oneNorm}[2][1]{
    \ensuremath{{}_{\phantom{#1}}\norm{\vec{#2}}_{#1}}
}
\newcommand{\oneNormNoVec}[2][1]{
    \ensuremath{{}_{\phantom{#1}}\norm{{#2}}_{#1}}
}

\newcommand{\sign}{\textrm{sign}}
\newcommand{\abs}[1]{\vert #1 \vert}

\newcolumntype{P}[1]{>{\centering\arraybackslash}p{#1}}

\newcommand{\mycontour}[1]{\contour{green}{\textbf{#1}}}

\usepackage[accepted]{icml2024}

\icmltitlerunning{A2Q+: Improving Accumulator-Aware Weight Quantization}

\begin{document}

\twocolumn[
\icmltitle{A2Q+: Improving Accumulator-Aware Weight Quantization}

\begin{icmlauthorlist}
\icmlauthor{Ian Colbert}{amd_san_diego}
\icmlauthor{Alessandro Pappalardo}{amd_dublin}
\icmlauthor{Jakoba Petri-Koenig}{amd_dublin}
\icmlauthor{Yaman Umuroglu}{amd_dublin}
\end{icmlauthorlist}

\icmlaffiliation{amd_san_diego}{AMD SW Technology Team, San Diego, California, USA}
\icmlaffiliation{amd_dublin}{AMD Research and Advanced Development, Dublin, Ireland}

\icmlcorrespondingauthor{Ian Colbert}{ian.colbert@amd.com}

\icmlkeywords{
    Accumulators,
    Machine Learning,
    Quantization,
    Deep Neural Networks,
    Inference Optimization}

\vskip 0.3in
]

\printAffiliationsAndNotice{}

\begin{abstract}
Quantization techniques commonly reduce the inference costs of neural networks by restricting the precision of weights and activations.
Recent studies show that also reducing the precision of the accumulator can further improve hardware efficiency at the risk of numerical overflow, which introduces arithmetic errors that can degrade model accuracy.
To avoid numerical overflow while maintaining accuracy, recent work proposed accumulator-aware quantization (A2Q)---{a quantization-aware training method that constrains model weights during training to safely use a target accumulator bit width during inference.}
{Although this shows promise, we demonstrate} that A2Q relies on an overly restrictive constraint and a sub-optimal weight initialization strategy that each introduce superfluous quantization error.
To address these shortcomings, we introduce: (1) {an improved} bound that alleviates accumulator constraints {without compromising overflow avoidance}; and (2) a new strategy for initializing quantized weights from pre-trained floating-point checkpoints.
We combine these contributions with weight normalization to introduce A2Q+.
We identify and characterize the various trade-offs that arise as a consequence of accumulator constraints and support our analysis with experiments that show A2Q+ significantly improves these trade-offs when compared to prior methods.
\end{abstract}

\vspace{-0.4cm}
\section{Introduction}
\label{sec:intro}

Quantizing neural network weights and activations to low-precision integers can drastically reduce the inference costs of multiplications.
However, the resulting products are commonly accumulated at high-precision and thus require high-precision additions and registers.
Recent studies show that reducing the standard $32$-bit accumulators to $16$ bits can yield a near-optimal $2\times$ increase in throughput and bandwidth efficiency on ARM processors~\cite{de2020quantization, xie2021overflow} and up to a $1.6\times$ reduction in resource utilization on FPGAs~\cite{colbert2023a2q}.
{However, exploiting such an optimization is highly non-trivial in practice as doing so also incurs a high risk of numerical overflow, which introduces arithmetic errors that can significantly degrade model accuracy~\cite{ni2021wrapnet}.}

To train quantized neural networks (QNNs) for low-precision accumulation,~\citealt{colbert2023a2q} recently proposed accumulator-aware quantization (A2Q).
While prior approaches had sought to either reduce the risk of numerical overflow~\cite{xie2021overflow, li2022downscaling, azamat2022squeezing} or mitigate its impact on model accuracy~\cite{ni2021wrapnet, blumenfeld2023towards}, A2Q circumvents arithmetic errors caused by numerical overflow by constraining {model} weights to restrict the range of {outputs}.
In doing so, A2Q provides state-of-the-art performance for low-precision accumulation with {guaranteed} overflow avoidance.

Our work contributes to this body of research by further improving the trade-off between accumulator bit width and model accuracy.
We show that A2Q relies on: (1) an overly restrictive $\ell_1$-norm bound that constrains QNNs more than necessary; and (2) a sub-optimal initialization strategy that forces QNNs to recover from superfluous quantization error.
In addressing these shortcomings, we establish a new state-of-the-art for low-precision accumulation with guaranteed overflow avoidance.
Our results show for the first time that ResNet50~\cite{he2016deep} can maintain $95\%$ of its baseline accuracy when trained on ImageNet~\cite{deng2009imagenet} to accumulate at 12 bits without overflow, resulting in a $+17\%$ improvement in test top-1 accuracy over A2Q.

Our contributions are four-fold:
(1) we introduce a new theoretical analysis for an improved $\ell_1$-norm bound that alleviates accumulator constraints without compromising overflow avoidance;
(2) we introduce a weight initialization strategy that minimizes the initial weight quantization error caused by accumulator constraints;
(3) we combine (1) and (2) with weight normalization~\cite{salimans2016weight} to introduce A2Q+ and show significant improvements in the trade-off between accumulator bit width and model accuracy;
and (4) we identify and characterize various trade-offs that arise as a consequence of accumulator constraints.

\section{Background and Related Work}
\label{sec:background}

\subsection{Low-Precision Accumulation}
\label{sec:low_prec_acc}

Neural network primitives are commonly executed as dot products consisting of numerous multiply-accumulate (MAC) operations.
During inference, the inputs to these dot products (\textit{i.e.}, the weights and activations) are increasingly being represented with lower precision {integers} to reduce the cost of multiplications; meanwhile, their products are still accumulated using high-precision additions.

The skew towards weight and activation quantization is in large part because the most commonly studied data formats in deep learning inference have required $8$ or more bits~\cite{wu2020integer, gholami2021survey}.
Because the cost of integer MACs scales quadratically with the bit widths of the weights and activations but linearly with that of the accumulator~\cite{horowitz20141, blott2018finn, hawks2021ps}, the cost of multiplications dwarfs that of additions in such paradigms.
However, with even lower precision data formats increasing in popularity~\cite{aggarwal2023post, wu2023understanding}, ignoring the accumulator to solely focus on low-precision weight and activation quantization will have diminishing returns.
{For example, \citealt{ni2021wrapnet} show that when constraining weights and activations to $3$-bit $\times$ $1$-bit multipliers, the cost of $32$-bit accumulation dominates that of multiplication, consuming nearly $75\%$ of the total power and $90\%$ of the total area of their MAC unit.
When reducing to an $8$-bit accumulator, they report 4$\times$ power savings and 5$\times$ area reduction.
In addition to power and area, recent work has also demonstrated savings in throughput and bandwidth utilization when reducing the accumulator bit width on general-purpose platforms~\cite{de2020quantization, xie2021overflow}.
{Both~\citealt{de2020quantization} and~\citealt{xie2021overflow} report a near-optimal $2\times$ increase in throughput on computer vision workloads when reducing the accumulator width from $32$ to $16$ bits on ARM processors.}}

Exploiting such an optimization in a principled manner is non-trivial in practice.
{The risk of numerical overflow increases exponentially as the accumulator bit width is reduced~\cite{colbert2023a2q}.
The resulting arithmetic errors can lead to catastrophic degradation in model accuracy if the accumulator is not large enough~\cite{ni2021wrapnet}.}

\subsection{Accumulator-Aware Quantization (A2Q)}
\label{sec:a2q_background}

Training neural networks with quantization in the loop is a useful means of recovering model accuracy lost to quantization errors~\cite{gholami2021survey, wu2020integer}.
The standard operators used to emulate quantization during training are built on top of uniform affine transformations that map high-precision values to low-precision ones.
We refer to the operators that perform these transformations as quantizers.
As given by Eq.~\ref{eq:standard_quantizer}, quantizers are commonly parameterized by zero-point~$z$ and scaling factor $s$.
Here, $z$ is an integer value that ensures that zero is exactly represented in the quantized domain, and $s$ is a strictly positive real scalar that corresponds to the resolution of the mapping.
Scaled values are commonly rounded to the nearest integers using half-way rounding, denoted by $\lfloor \cdot \rceil$, and elements that exceed the largest supported values in the quantized domain are clipped to $n$ and $p$, which depend on the target bit width~$b$.
We assume $n=-2^{b-1}$ and $p=2^{b-1} - 1$ when signed, and $n=0$ and $p=2^b - 1$ when unsigned.
\begin{equation}
Q(\vec{w}) := s \cdot \left( \textrm{clip}(\left\lfloor \frac{\vec{w}}{s} \right\rceil + z; n, p) - z \right) \label{eq:standard_quantizer}
\end{equation}
One approach to training QNNs for low-precision accumulation is to mitigate the impact of numerical overflow on model accuracy during QAT.
To do so, researchers have sought to either tune scale factors to control overflow rates~\cite{xie2021overflow, azamat2022squeezing, li2022downscaling} or train QNNs to be robust to wraparound arithmetic~\cite{ni2021wrapnet, blumenfeld2023towards}.
However, empirical estimates of overflow rely on \textit{a priori} knowledge of the input distribution, which is impractical to assume in many real-world use cases and can even introduce vulnerabilities~\cite{baier2019challenges}.
Thus, as an alternative,~\citealt{colbert2023a2q} proposed accumulator-aware quantization (A2Q) to directly train QNNs to use low-precision accumulators during inference without any risk of numerical overflow.

A2Q guarantees overflow avoidance by constraining the $\ell_1$-norm of weights to restrict the range of dot product outputs.
To accomplish this, \citealt{colbert2023a2q} introduce a quantizer inspired by weight normalization~\cite{salimans2016weight} that re-parameterizes weights $\vec{w}$ into vector $\vec{v}$ and scalar $g$ such that $\vec{w} = g \cdot \vec{v} / \Vert \vec{v} \Vert_1$.
This allows the $\ell_1$-norm of $\vec{w}$ to be learned as an independent parameter since $g = \Vert \vec{w} \Vert_1$.
{To avoid numerical overflow during inference, \citealt{colbert2023a2q} 
constrain $g$ according to a derived upper bound $T$ so that $\Vert Q(\vec{w}) \Vert_1 \leq T$, as further discussed in Section~\ref{sec:a2q_v2}.}
The resulting quantizer is defined as:
\begin{align}
Q(\vec{w}) :=  & ~s \cdot \textrm{clip} \left(\left\lfloor \frac{\vec{w}}{s} \right\rfloor ; n, p\right) \label{eq:a2q_v1_quantizer} \\
\textrm{where } \vec{w} = & ~\frac{\vec{v}}{\oneNorm{v}} \cdot \min(g, T) \label{eq:a2q_v1_clip} \\
\textrm{and } T = &~ s \cdot \frac{2^{P - 1} - 1}{2^{N - \mathbbm{1}_\text{signed}(\bm{x})}} \label{eq:a2q_v1_T}
\end{align}
Here, $P$ denotes the target accumulator bit width, $N$ denotes the input activation bit width, and $\mathbbm{1}_\text{signed}(\bm{x})$ is an indicator function that returns $1$ when input activations are signed and $0$ when unsigned.
Unlike in Eq.~\ref{eq:standard_quantizer}, scaled weights are rounded towards zero~\cite{loroch2017tensorquant}, denoted by $\left\lfloor \cdot \right\rfloor$, to prevent any upward rounding that may cause $\Vert Q(\vec{w})\Vert_1$ to increase past the derived upper bound $T$.
Each output channel is assumed to have its own accumulator so $g$ is independently defined and constrained per-channel.

\section{A2Q+}
\label{sec:a2q}

Let weights $\vec{q}$ be a $K$-dimensional vector of $M$-bit integers, and let
$\mathbb{Z}^K_N$ denote the set of all $K$-dimensional vectors of $N$-bit integers.
When accumulating the dot product of $\vec{q}$ by any $\vec{x} \in \mathbb{Z}^K_N$ into a signed $P$-bit register, \citealt{colbert2023a2q} show that one can avoid overflow if $\vec{q}$ satisfies:
\begin{equation}
\oneNorm{q} \leq \frac{2^{P - 1} - 1}{2^{N - \mathbbm{1}_\text{signed}(\bm{x})}}
\label{eq:a2q_v1}
\end{equation}
Irrespective of weight bit width $M$, Eq.~\ref{eq:a2q_v1} establishes an upper bound on the $\ell_1$-norm of $\vec{q}$ as a function of accumulator bit width~$P$ and activation bit width~$N$.
For fixed $N$, reducing $P$ exponentially tightens the constraint on~$\Vert \vec{q} \Vert_1$, which restricts the range of the weights by pulling them towards zero.
\citealt{colbert2023a2q} demonstrate that learning under such a constraint introduces a trade-off in the form of a Pareto frontier, where reducing the accumulator bit width invariably limits model accuracy within a fixed quantization design space.
We observe that this bound also introduces a non-trivial trade-off between activation bit width~$N$ and model accuracy.
Reducing the precision of the activations alleviates pressure on $\Vert \vec{q} \Vert_1$, which becomes more significant as $P$ is reduced.
However, aggressive discretization of intermediate activations can significantly hurt model accuracy~\cite{wu2020integer, gholami2021survey}.
In Section~\ref{sec:per_acc_dt}, we show that balancing this trade-off results in the Pareto-optimal activation bit width~$N$ decreasing with $P$.

Rather than tackling this balancing act (which is an intriguing problem for future work), our work extends the approach of A2Q to directly improves these trade-offs.
We demonstrate that A2Q relies on an overly restrictive constraint and a sub-optimal weight initialization strategy that each introduce superfluous quantization errors.
In Section~\ref{sec:experimental_results}, we show that minimizing these errors ultimately leads to improved model accuracy as the accumulator bit width is reduced.

\subsection{Improved $\ell_1$-norm Bound via Zero-Centering}
\label{sec:derivation}

Let the closed interval $[a, b]$ denote the representation range of a signed $P$-bit register.
To avoid overflow when accumulating $\vec{x}^T \vec{q}$ into this register, the dot product output needs to fall within $[a, b]$ for any $\vec{x} \in \mathbb{Z}^K_N$.
Without loss of generality, we assume a two's complement representation in our work, where $[a,b] = [-2^{P-1}, 2^{P-1}-1]$, as is common practice~\cite{wu2020integer, gholami2021survey}.

\citealt{colbert2023a2q} approach this task by constraining the magnitude of $\vec{x}^T \vec{q}$ such that $\vert \vec{x}^T \vec{q} \vert \leq 2^{P-1} - 1$.
They use worst-case values for~$\vec{x}$ to derive the upper bound on $\Vert \vec{q} \Vert_1$ given by Eq.~\ref{eq:a2q_v1}.
Note that this bound can similarly be constructed via H\"older's inequality~\cite{hardy1952inequalities} as shown in Eq.~\ref{eq:holders_inequality}, where $\Vert \vec{x} \Vert_{\infty} = 2^{N - \mathbbm{1}_\text{signed}(\bm{x})}$.
\begin{equation}
\vert \vec{x}^T \vec{q} \vert \leq \Vert \vec{x} \Vert_{\infty} \Vert \vec{q} \Vert_1 \leq 2^{P - 1} - 1
\label{eq:holders_inequality}
\end{equation}
{Note that this bound has two shortcomings: (1) it does not make full use of the representation range of the accumulator, which becomes increasingly important as its bit width $P$ is reduced; and (2) it depends on the sign of $\vec{x}$, which tightens the constraint by 2$\times$ when $\mathbbm{1}_\text{signed}(\bm{x})=0$.
In this work, we resolve both of these.}
We find that zero-centering our weight vector {such that $\textstyle \sum_i w_i = 0$} yields a favorable property, as formally presented in the following proposition. \\

\begin{figure}[t!]
\centering
\includegraphics[width=0.85\linewidth]{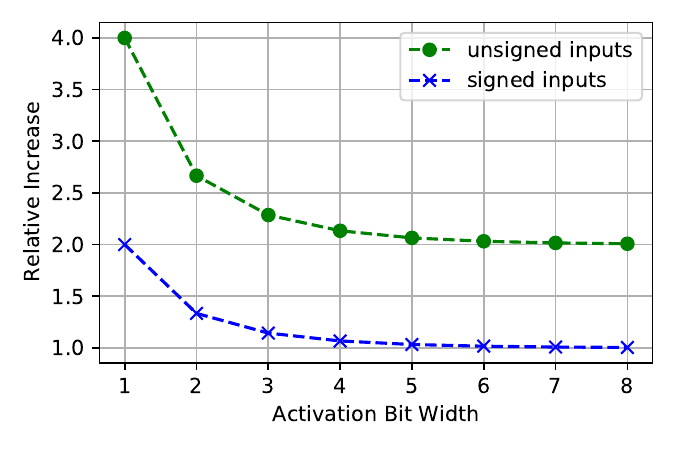}
\caption{We visualize Eq.~\ref{eq:ratio_new_to_old} for both signed (\textcolor{blue}{\textbf{blue crosses}}) and unsigned (\textcolor{ForestGreen}{\textbf{green circles}}) integers to show the relative increase in $\ell_1$-norm budget that our new bound (Eq.~\ref{eq:new_bound_prop_1}) gives to $\vec{q}$ when compared to the standard A2Q bound (Eq.~\ref{eq:a2q_v1}).}
\label{fig:v1_vs_v2}
\end{figure}

\begin{restatable}{proposition}{propOne}
Let $\vec{x}$ be a $K$-dimensional vector of $N$-bit integers such that the value of the $i$-th element $x_i$ lies within the closed interval $[c,d]$ and $d - c = 2^N - 1$.
Let $\vec{q}$ be a $K$-dimensional vector of signed integers centered at zero such that $\textstyle \sum_i q_i = 0$.
To guarantee overflow avoidance when accumulating the result of $\vec{x}^T \vec{q}$ into a {signed} $P$-bit register, it is sufficient that the $\ell_1$-norm of $\vec{q}$ satisfies:
\begin{equation}
\oneNorm{q} \leq \frac{2^P - 2}{2^N - 1}
\label{eq:new_bound_prop_1}
\end{equation}
\label{prop:upper_bound_on_l1_norm}
\end{restatable}
\vspace{-0.3cm}
The proof of this proposition is provided in Appendix~\ref{appendix:prop_1_proof}.
{It is important to note that our new bound (Eq.~\ref{eq:new_bound_prop_1}) utilizes the full representation range of the accumulator and is agnostic to the sign of the input data.}
Furthermore, when compared to the original bound (Eq.~\ref{eq:a2q_v1}), ours is greater by a factor of:
\begin{equation}
\frac{2^{N + 1 - \mathbbm{1}_\text{signed}(\bm{x})}}{2^N - 1}
\label{eq:ratio_new_to_old}
\end{equation}
In Fig.~\ref{fig:v1_vs_v2}, we visualize this relationship as a function of activation bit width $N$.
As implied by Eq.~\ref{eq:ratio_new_to_old}, the impact of our bound increases as the activation bit width is reduced, with the greatest significance in sub-4-bit quantization scenarios.
{In fact, our bound yields up to a 4$\times$ increase in the $\ell_1$-norm budget afforded to $\vec{q}$ when input activations are unsigned (\textit{i.e.}, $\mathbbm{1}_\text{signed}(\bm{x})=0$), and up to 2$\times$ when they are signed (\textit{i.e.}, $\mathbbm{1}_\text{signed}(\bm{x})=1$).}
In Section~\ref{sec:experimental_results}, we show that this increased freedom significantly improves model accuracy.

\subsection{Improved Initialization via Euclidean Projections}
\label{sec:l1_proj_init}

By re-parameterizing weight vector $\vec{w}$ as
defined below in Eq.~\ref{eq:weight_norm_reparam},
A2Q introduces two new parameters to initialize: $g$ and $\vec{v}$.
However, because $\vec{w}$ is a {function of these learned parameters}, it can no longer be directly initialized from a pre-trained floating-point checkpoint.
\begin{equation}
\vec{w} = g \cdot \frac{\vec{v}}{\oneNorm{v}}
\label{eq:weight_norm_reparam}
\end{equation}
One could trivially initialize $\vec{v}$ to be the pre-trained floating-point weight vector $\vec{w}_{\textrm{float}}$ and $g$ to be its $\ell_1$-norm, where $\vec{v} = \vec{w}_{\textrm{float}}$ and $g = \Vert \vec{w}_{\textrm{float}} \Vert_1$, making~$\vec{w} = \vec{w}_{\textrm{float}}$.
However, A2Q clips $g$ according to $T$ in Eq.~\ref{eq:a2q_v1_clip}.
{As a consequence, we observe that na\"ively initializing $g$ and $\vec{v}$ according to $\vec{w}_{\textrm{float}}$ introduces excessive weight quantization error when $\Vert \vec{w}_{\textrm{float}} \Vert_1 > T$ (see Appendix~\ref{appendix:l1_proj_ablation}).
Thus, we aim to minimize weight quantization error at initialization.

We formulate our objective as a  projection task described by the constrained convex optimization problem in Eq.~\ref{eq:convex_opt_prob}.
Here, the optimal initialization $\vec{v}^*$ minimizes the weight quantization error while satisfying the $\ell_1$-norm accumulator constraint on the re-scaled quantized weights $Q(\vec{w})$.
\begin{align}
\vec{v}^* =&~ \min_{\vec{v}} \frac{1}{2} \Vert Q(\vec{w}) - \vec{w}_{\textrm{float}} \Vert^2_2 
\label{eq:convex_opt_prob} \\
 \textrm{subject to} &~\Vert Q(\vec{w}) \Vert_1 \leq T \\
 \textrm{where} &~\vec{w} = g \cdot \frac{\vec{v}}{\oneNorm{v}}
\end{align}
To solve this optimization problem, we exploit the round-to-zero operator, which ensures that the magnitude of any weight $w_i$ is always greater than or equal to that of its quantized counterpart $Q(w_i)$, or more formally $\vert Q(w_i) \vert \leq \vert w_i \vert$ for all $i$.}
This allows us to solely focus on initializing $\vec{v}$ such that $\Vert \vec{v} \Vert_1 \leq T$ and then initialize $g$ such that $g = \Vert \vec{v} \Vert_1$.
Thus, we can simplify our optimization problem to:
\begin{align}
\vec{v}^* =&~ \min_{\vec{v}} \frac{1}{2} \Vert \vec{v} - \vec{w}_{\textrm{float}} \Vert^2_2 \label{eq:new_conv_opt_prob} \\
\textrm{subject to} &~\Vert \vec{v} \Vert_1 \leq T
\end{align}
It is important to first note that if $\Vert \vec{w}_{\textrm{float}} \Vert_1 \leq T$, then the optimal solution to Eq.~\ref{eq:new_conv_opt_prob} is trivially $\vec{v}^* = \vec{w}_{\textrm{float}}$.
In addition, when $\Vert \vec{w}_{\textrm{float}} \Vert_1 > T$, the optimal solution $\vec{v}^*$ lies on the boundary of the constrained set such that $\Vert \vec{v}^* \Vert_1 = T$.
This allows us leverage the optimal solution derived in~\citealt{duchi2008efficient}, which efficiently projects $\vec{w}_{\textrm{float}}$ onto an $\ell_1$-ball of radius $T$ using Eq.~\ref{eq:l1_proj_init_def}.
\begin{align}
\vec{v}^* &~= \sign(\vec{w}_{\textrm{float}})\left( \abs{\vec{w}_{\textrm{float}}}- \theta \right)_{+} \label{eq:l1_proj_init_def}
\end{align}
Here, $(\cdot)_{+}$ denotes the rectified linear unit, which zeroes out all negative values, and $\theta$ is a Lagrangian scalar derived from the optimal solution.
{We direct the reader to~\citealt{duchi2008efficient} for the associated proofs and derivations.}

\subsection{Constructing A2Q+}
\label{sec:a2q_v2}

Similar to A2Q, our quantizer is inspired by weight normalization~\cite{salimans2016weight} and leverages the re-parameterization given in Eq.~\ref{eq:weight_norm_reparam}.
However, unlike A2Q, we are unable to simply constrain scalar parameter~$g$ according to Eq.~\ref{eq:new_bound_prop_1} because Prop.~\ref{prop:upper_bound_on_l1_norm} relies on the assumption that $Q(\vec{w})$ is zero-centered such that $\textstyle \sum_i Q(w_i)= 0$, which is not inherently guaranteed.

Enforcing such a zero-centering constraint on a vector of integers is non-trivial in practice.
Emulating quantization during QAT adds to this complexity as integer-quantized weights $Q(\vec{w})$ are a function of floating-point counterpart~$\vec{w}$.
However, A2Q is able to guarantee the $\ell_1$-norm constraint on $Q(\vec{w})$ by constraining norm parameter $g$, then rounding the scaled floating-point weights $\vec{w} / s$ towards zero, which ensures that $\Vert Q(\vec{w}) \Vert_1 \leq \Vert \vec{w} \Vert_1$.
We similarly exploit this property to enforce our zero-centering constraint, as formally articulated in the following proposition. \\

\begin{restatable}{proposition}{propTwo}
{Let $\vec{x}$ be a vector of $N$-bit integers such that the $i$-th element $x_i$ lies within the closed interval $[c,d]$ and $d-c = 2^N - 1$.
Let $\vec{w}$ be a zero-centered vector such that~$\textstyle \sum_i w_i = 0$.
Let $Q(\vec{w})$ be a symmetric quantizer parameterized in the form of Eq.~\ref{eq:a2q_v1_quantizer}, where~$Q(\vec{w}) = s \cdot \vec{q}$ for strictly positive scaling factor $s$ and integer-quantized weight vector $\vec{q}$.
Given that $\textrm{\emph{sign}}(s \cdot q_i) = \textrm{\emph{sign}}({w}_i)$ and $\vert s \cdot q_i \vert \leq \vert {w_i} \vert$ for all $i$, then $\vec{x}^T \vec{q}$ can be safely accumulated into a signed $P$-bit register without overflow if $\vec{w} / s$ satisfies all necessary conditions for such a constraint.}
\label{prop:n2}
\end{restatable}

\begin{figure*}[t!]
\includegraphics[width=\linewidth]{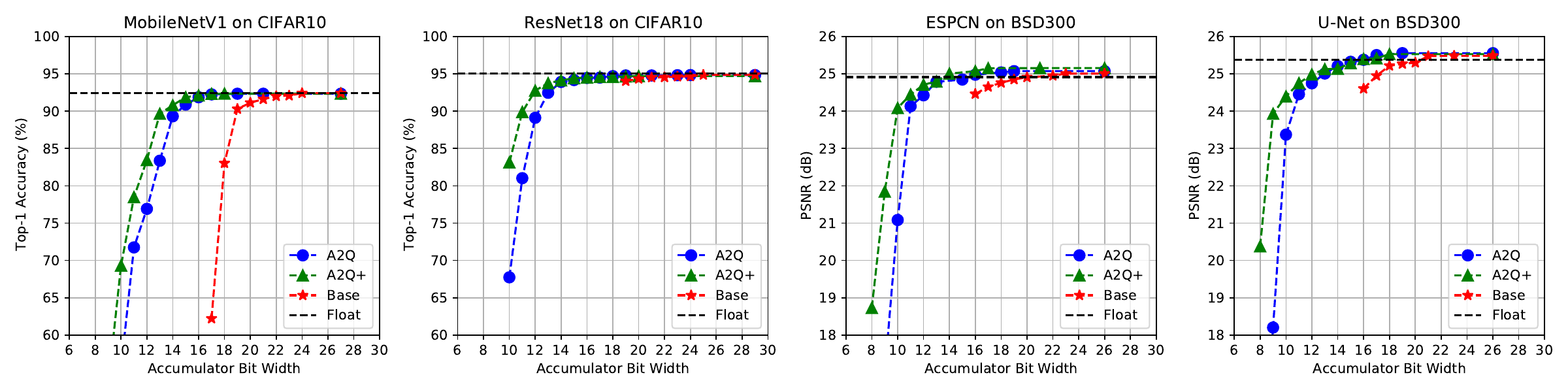}
\caption{We visualize the trade-off between accumulator bit width and model accuracy using Pareto frontier. We observe that A2Q+ (\textcolor{ForestGreen}{green triangles}) dominates both A2Q (\textcolor{blue}{blue circles}) and the baseline QAT (\textcolor{red}{red stars}) in all benchmarks.}
\label{fig:per_acc_dt}
\end{figure*}

The proof of this proposition, as well as a formal definition of the necessary accumulator constraint conditions, is provided in Appendix~\ref{appendix:prop_2_proof}.
{Based on Prop.~\ref{prop:n2}, we are able to enforce our zero-centering constraint on~$\vec{w}$ without compromising overflow avoidance, assuming we maintain a symmetric quantizer that rounds towards zero.
However, rather than directly zero-centering~$\vec{w}$  when leveraging the re-parameterization given by Eq.~\ref{eq:weight_norm_reparam}, we enforce our constraint on~$\vec{v}$ so as to control its $\ell_1$-norm.
Given that $\gamma = g / \Vert \vec{v} \Vert_1$ and $\gamma \vec{v} = \vec{w}$, it follows that $\textstyle \sum_i w_i = 0$ when $\textstyle \sum_i v_i = 0$.
Thus, we construct our quantizer as follows:}
\begin{align}
Q(\vec{w}) :=  & ~s \cdot \textrm{clip} \left(\left\lfloor \frac{\vec{w}}{s} \right\rfloor; n, p \right) \label{eq:a2q_v2_quantizer} \\
\textrm{where } \vec{w} = & ~\frac{\vec{v} - \mu_{\vec{v}}}{\oneNormNoVec{\vec{v} - \mu_{\vec{v}}}} \cdot \min(g, T_{+}) \\
\textrm{and } {\mu}_{\vec{v}} = &~ \frac{1}{K} \sum_{i=1}^K v_i \\
\textrm{and } T_{+} = & ~s \cdot \frac{2^P - 2}{2^N - 1}
\end{align}
To maintain a symmetric quantizer, we eliminate the zero points in our mapping such that~$z=0$.
We also use an exponential parameterization of both the scaling factor $s = 2^d$ and norm parameter $g = 2^t$, where $d$ and $t$ are defined per-output channel and learned through gradient descent.
Note that norm parameter $g$ is clipped to $T_+$, which is our new upper bound defined in Eq.~\ref{eq:new_bound_prop_1} scaled by $s$ to ensure $\Vert \vec{w} / s \Vert_1 \leq T_+$.
Our scaled floating-point weights are then rounded towards zero, denoted by $\lfloor \cdot \rfloor$.
The rounded weights are then clipped and re-scaled.
When updating learnable parameters throughout training, we use the straight-through estimator~\cite{bengio2013estimating} to allow gradients to permeate the rounding function, where $\nabla_x \lfloor x \rfloor = 1$ everywhere and $\nabla_x$ denotes the gradient with respect to $x$.

Extending our Euclidean projection-based initialization strategy to A2Q+ is non-trivial in practice.
In such a scenario, our optimization problem is instead subject to the following constraint:  $\Vert \vec{v} - \mu_{\vec{v}} \Vert_1 \leq T_+$.
Furthermore, the optimal solution derived by~\citealt{duchi2008efficient} requires each non-zero component of the optimal solution $\vec{v}^*$ to share the same sign as its counterpart in $\vec{w}_{\textrm{float}}$, which is not inherently guaranteed due to our zero-centering constraint.
Therefore, we initialize all A2Q+ networks using the A2Q initialization in the scope of this work.
In practice, we observe this still significantly reduces initial weight quantization error for A2Q+ networks.
In Appendix~\ref{appendix:l1_proj_ablation}, we provide a deeper investigation for both A2Q and A2Q+ networks.

\section{Experimental Results}
\label{sec:experimental_results}

\textbf{Models \& Datasets.} Throughout our experiments, we focus on two computer vision tasks: image classification and single-image super resolution.
In Section~\ref{sec:per_acc_dt}, we evaluate MobileNetV1~\cite{howard2017mobilenets} and ResNet18~\cite{he2016deep} trained on the CIFAR-10 dataset~\cite{krizhevsky2009learning} for image classification, and ESPCN~\cite{shi2016real} and U-Net~\cite{ronneberger2015u} trained on the BSD300 dataset~\cite{martin2001database} for super resolution.
In Section~\ref{sec:more_models}, we evaluate larger image classification benchmarks, namely ResNet18, ResNet34, and ResNet50 trained on the ImageNet-1K dataset~\cite{deng2009imagenet}.

\textbf{Quantization Design Space.} Following the experiments of~\citealt{colbert2023a2q}, we constrain our quantization design space to uniform-precision models such that every hidden layer has the same weight, activation, and accumulator bit width, respectively denoted as $M$, $N$, and $P$.
Our experiments consider $3$- to $8$-bit integers for both weights and activations, extending the quantization design space of ~\citealt{colbert2023a2q} by $4 \times$.
For each of the $64$ weight and activation combinations, we calculate the most conservative accumulator bit width for each model using Eq.~\ref{eq:data_type_bound}, as derived by~\citealt{colbert2023a2q}.
Here, $\lceil \cdot \rceil$ denotes ceiling rounding and $K^* = \argmax_{K_l} \{ K_l \}^L_{l=1}$, where $K_l$ is the dot product size of layer $l$ in a network with $L$ layers.
{We calculate $P^*$ for each unique $(M,N)$ combination and evaluate up to a $10$-bit reduction in accumulator bit width, creating a total of $640$ unique configurations per model.
We repeat each experiment 3 times using different random seeds.}
\begin{align}
P^* = & ~\left\lceil \alpha + \phi(\alpha)  + 1 \right\rceil \label{eq:data_type_bound} \\
\alpha = & ~\log_2(K^*) + N + M - 1 - \mathbbm{1}_{\rm signed}(\vec{x}) \\
\phi(\alpha) = & ~\log_2(1 + 2^{-\alpha})
\end{align}
We implement A2Q+ in PyTorch~\cite{paszke2019pytorch} using v0.10 of the Brevitas quantization library~\cite{brevitas} and leverage their implementations of A2Q and baseline QAT methods for benchmarking.
We include all training details and hyperparameters in Appendix~\ref{appendix:hyperparams}.

\begin{table*}[t!]
\vspace{-0.3cm}
\caption{We provide the test top-1 accuracy and quantization configuration for some of the Pareto-optimal image classification models that form a section of the frontiers visualized in Fig.~\ref{fig:per_acc_dt}. We highlight the Pareto-dominant points in \textcolor{ForestGreen}{green}.}
\centering
\resizebox{\textwidth}{!}{\begin{tabular}{|c|cccccc||cccccc|}
\hline
\multirow{3}{*}{$\textbf{P}$} & \multicolumn{6}{c||}{\thead{\textbf{MobileNetV1} \\ (Float: 92.43\%)}} & \multicolumn{6}{c|}{\thead{\textbf{ResNet18} \\ (Float: 95.00\%) }}\\ \cline{2-13} 
 & \multicolumn{2}{c|}{\textbf{Base}} & \multicolumn{2}{c|}{\textbf{A2Q}} & \multicolumn{2}{c||}{\textbf{A2Q+}}    & \multicolumn{2}{c|}{\textbf{Base}} & \multicolumn{2}{c|}{\textbf{A2Q}} & \multicolumn{2}{c|}{\textbf{A2Q+}}    \\ \cline{2-13} 
 & \multicolumn{1}{c|}{\textbf{Top-1}} & \multicolumn{1}{c|}{\textbf{$(M, N)$}} & \multicolumn{1}{c|}{\textbf{Top-1}} & \multicolumn{1}{c|}{\textbf{$(M, N)$}} & \multicolumn{1}{c|}{\textbf{Top-1}} & \textbf{$(M, N)$} & \multicolumn{1}{c|}{\textbf{Top-1}} & \multicolumn{1}{c|}{\textbf{$(M, N)$}} & \multicolumn{1}{c|}{\textbf{Top-1}} & \multicolumn{1}{c|}{\textbf{$(M, N)$}} & \multicolumn{1}{c|}{\textbf{Top-1}} & \textbf{$(M, N)$} \\ \hline \hline
11 & \multicolumn{1}{c|}{-} & \multicolumn{1}{c|}{-} & \multicolumn{1}{c|}{71.8\%} & \multicolumn{1}{c|}{(5,4)} & \multicolumn{1}{c|}{\mycontour{78.5\%}} & {\mycontour{(4,5)}} & \multicolumn{1}{c|}{-} & \multicolumn{1}{c|}{-} & \multicolumn{1}{c|}{81.0\%} & \multicolumn{1}{c|}{(3,3)} & \multicolumn{1}{c|}{\mycontour{89.9\%}} & \mycontour{(3,3)} \\ \hline
12 & \multicolumn{1}{c|}{-} & \multicolumn{1}{c|}{-} & \multicolumn{1}{c|}{76.9\%} & \multicolumn{1}{c|}{(4,5)} & \multicolumn{1}{c|}{\mycontour{83.5\%}} & \mycontour{(4,6)} & \multicolumn{1}{c|}{-} & \multicolumn{1}{c|}{-} & \multicolumn{1}{c|}{89.1\%} & \multicolumn{1}{c|}{(3,3)} & \multicolumn{1}{c|}{\mycontour{92.8\%}} & \mycontour{(3,3)} \\ \hline
13 & \multicolumn{1}{c|}{-} & \multicolumn{1}{c|}{-} & \multicolumn{1}{c|}{83.4\%} & \multicolumn{1}{c|}{(6,5)} & \multicolumn{1}{c|}{\mycontour{89.7\%}} & \mycontour{(3,6)} & \multicolumn{1}{c|}{-} & \multicolumn{1}{c|}{-} & \multicolumn{1}{c|}{92.5\%} & \multicolumn{1}{c|}{(3,3)} & \multicolumn{1}{c|}{\mycontour{93.8\%}} & \mycontour{(4,3)}  \\ \hline
14 & \multicolumn{1}{c|}{-} & \multicolumn{1}{c|}{-} & \multicolumn{1}{c|}{89.3\%} & \multicolumn{1}{c|}{(3,6)} & \multicolumn{1}{c|}{\mycontour{90.8\%}} & \mycontour{(3,7)} & \multicolumn{1}{c|}{-} & \multicolumn{1}{c|}{-} & \multicolumn{1}{c|}{93.9\%} & \multicolumn{1}{c|}{(4,3)} & \multicolumn{1}{c|}{\mycontour{94.1\%}} &  \mycontour{(4,4)} \\ \hline
15 & \multicolumn{1}{c|}{-} & \multicolumn{1}{c|}{-} & \multicolumn{1}{c|}{90.9\%} & \multicolumn{1}{c|}{(4,7)} & \multicolumn{1}{c|}{\mycontour{91.9\%}} & \mycontour{(4,7)} & \multicolumn{1}{c|}{-} & \multicolumn{1}{c|}{-} & \multicolumn{1}{c|}{94.2\%} & \multicolumn{1}{c|}{(4,4)} & \multicolumn{1}{c|}{\mycontour{94.4\%}} & \mycontour{(5,5)}  \\ \hline
16 & \multicolumn{1}{c|}{-} & \multicolumn{1}{c|}{-} & \multicolumn{1}{c|}{91.9\%} & \multicolumn{1}{c|}{(4,7)} & \multicolumn{1}{c|}{\mycontour{92.1\%}} & \mycontour{(5,8)} & \multicolumn{1}{c|}{-} & \multicolumn{1}{c|}{-} & \multicolumn{1}{c|}{94.4\%} & \multicolumn{1}{c|}{(4,4)} & \multicolumn{1}{c|}{\mycontour{94.5\%}} & \mycontour{(3,5)}  \\ \hline
17 & \multicolumn{1}{c|}{62.2\%} & \multicolumn{1}{c|}{(3,3)} & \multicolumn{1}{c|}{92.2\%} & \multicolumn{1}{c|}{(4,8)} & \multicolumn{1}{c|}{\mycontour{92.3\%}} & \mycontour{(6,8)} & \multicolumn{1}{c|}{-} & \multicolumn{1}{c|}{-} & \multicolumn{1}{c|}{94.5\%} & \multicolumn{1}{c|}{(5,5)} & \multicolumn{1}{c|}{\mycontour{94.6\%}} & \mycontour{(6,6)} \\ \hline
18 & \multicolumn{1}{c|}{83.0\%} & \multicolumn{1}{c|}{(3,4)} & \multicolumn{1}{c|}{{92.2\%}} & \multicolumn{1}{c|}{{(4,8)}} & \multicolumn{1}{c|}{\mycontour{92.4\%}} & \mycontour{(6,8)} & \multicolumn{1}{c|}{-} & \multicolumn{1}{c|}{-} & \multicolumn{1}{c|}{\mycontour{94.7\%}} & \multicolumn{1}{c|}{\mycontour{(4,5)}} & \multicolumn{1}{c|}{94.6\%} & (6,7) \\ \hline
19 & \multicolumn{1}{c|}{90.3\%} & \multicolumn{1}{c|}{(3,5)} & \multicolumn{1}{c|}{92.3\%} & \multicolumn{1}{c|}{(6,8)} & \multicolumn{1}{c|}{\mycontour{92.4\%}} & \mycontour{(6,8)} & \multicolumn{1}{c|}{94.0\%} & \multicolumn{1}{c|}{(3,3)} & \multicolumn{1}{c|}{\mycontour{94.8\%}} & \multicolumn{1}{c|}{\mycontour{(4,7)}} & \multicolumn{1}{c|}{94.7\%} & (6,8) \\ \hline
20 & \multicolumn{1}{c|}{91.1\%} & \multicolumn{1}{c|}{(3,6)} & \multicolumn{1}{c|}{{92.3\%}} & \multicolumn{1}{c|}{{(6,8)}} & \multicolumn{1}{c|}{\mycontour{92.4\%}} & \mycontour{(8,8)} & \multicolumn{1}{c|}{94.3\%} & \multicolumn{1}{c|}{(3,4)} & \multicolumn{1}{c|}{\mycontour{94.8\%}} & \multicolumn{1}{c|}{\mycontour{(4,7)}} & \multicolumn{1}{c|}{94.7\%} & (8,8) \\ \hline
\end{tabular}}
\label{tbl:pareto_cifar10}
\end{table*}

\begin{table*}[t!]
\vspace{-0.3cm}
\caption{We provide the test peak signal-to-noise ratio (PSNR) and quantization configuration for some of the Pareto-optimal super resolution models that form a section of the frontiers visualized in Fig.~\ref{fig:per_acc_dt}. We highlight the Pareto-dominant points in \textcolor{ForestGreen}{green}.}
\centering
\resizebox{\textwidth}{!}{\begin{tabular}{|c|cccccc||cccccc|}
\hline
\multirow{3}{*}{$\textbf{P}$} & \multicolumn{6}{c||}{\thead{\textbf{ESPCN} \\ (Float: 24.91)}} & \multicolumn{6}{c|}{\thead{\textbf{U-Net} \\ (Float: 25.37) }}\\ \cline{2-13} 
 & \multicolumn{2}{c|}{\textbf{Base}} & \multicolumn{2}{c|}{\textbf{A2Q}} & \multicolumn{2}{c||}{\textbf{A2Q+}}    & \multicolumn{2}{c|}{\textbf{Base}} & \multicolumn{2}{c|}{\textbf{A2Q}} & \multicolumn{2}{c|}{\textbf{A2Q+}}    \\ \cline{2-13} 
 & \multicolumn{1}{c|}{\textbf{PSNR}} & \multicolumn{1}{c|}{\textbf{$(M, N)$}} & \multicolumn{1}{c|}{\textbf{PSNR}} & \multicolumn{1}{c|}{\textbf{$(M, N)$}} & \multicolumn{1}{c|}{\textbf{PSNR}} & \textbf{$(M, N)$} & \multicolumn{1}{c|}{\textbf{PSNR}} & \multicolumn{1}{c|}{\textbf{$(M, N)$}} & \multicolumn{1}{c|}{\textbf{PSNR}} & \multicolumn{1}{c|}{\textbf{$(M, N)$}} & \multicolumn{1}{c|}{\textbf{PSNR}} & \textbf{$(M, N)$} \\ \hline \hline
9 & \multicolumn{1}{c|}{-} & \multicolumn{1}{c|}{-} & \multicolumn{1}{c|}{17.0} & \multicolumn{1}{c|}{(4,3)} & \multicolumn{1}{c|}{\mycontour{21.9}} & \multicolumn{1}{c||}{\mycontour{(3,3)}} & \multicolumn{1}{c|}{-} & \multicolumn{1}{c|}{-} & \multicolumn{1}{c|}{18.2} & \multicolumn{1}{c|}{(5,3)} & \multicolumn{1}{c|}{\mycontour{23.9}} & \mycontour{(5,3)} \\ \hline
10 & \multicolumn{1}{c|}{-} & \multicolumn{1}{c|}{-} & \multicolumn{1}{c|}{21.1} & \multicolumn{1}{c|}{(4,3)} & \multicolumn{1}{c|}{\mycontour{24.1}} & \multicolumn{1}{c||}{\mycontour{(4,3)}} & \multicolumn{1}{c|}{-} & \multicolumn{1}{c|}{-} & \multicolumn{1}{c|}{23.4} & \multicolumn{1}{c|}{(4,3)} & \multicolumn{1}{c|}{\mycontour{24.4}} & \mycontour{(3,3)} \\ \hline
11 & \multicolumn{1}{c|}{-} & \multicolumn{1}{c|}{-} & \multicolumn{1}{c|}{24.1} & \multicolumn{1}{c|}{(6,3)} & \multicolumn{1}{c|}{\mycontour{24.4}} & {\mycontour{(4,3)}} & \multicolumn{1}{c|}{-} & \multicolumn{1}{c|}{-} & \multicolumn{1}{c|}{24.5} & \multicolumn{1}{c|}{(6,3)} & \multicolumn{1}{c|}{\mycontour{24.8}} & \mycontour{(5,4)} \\ \hline
12 & \multicolumn{1}{c|}{-} & \multicolumn{1}{c|}{-} & \multicolumn{1}{c|}{24.4} & \multicolumn{1}{c|}{(7,3)} & \multicolumn{1}{c|}{\mycontour{24.7}} & \mycontour{(4,4)} & \multicolumn{1}{c|}{-} & \multicolumn{1}{c|}{-} & \multicolumn{1}{c|}{24.7} & \multicolumn{1}{c|}{(4,4)} & \multicolumn{1}{c|}{\mycontour{25.0}} & \mycontour{(3,5)} \\ \hline
13 & \multicolumn{1}{c|}{-} & \multicolumn{1}{c|}{-} & \multicolumn{1}{c|}{\mycontour{24.8}} & \multicolumn{1}{c|}{\mycontour{(7,4)}} & \multicolumn{1}{c|}{\mycontour{24.8}} & \mycontour{(5,5)} & \multicolumn{1}{c|}{-} & \multicolumn{1}{c|}{-} & \multicolumn{1}{c|}{25.0} & \multicolumn{1}{c|}{(8,4)} & \multicolumn{1}{c|}{\mycontour{25.1}} & \mycontour{(7,5)}  \\ \hline
14 & \multicolumn{1}{c|}{-} & \multicolumn{1}{c|}{-} & \multicolumn{1}{c|}{24.8} & \multicolumn{1}{c|}{(7,4)} & \multicolumn{1}{c|}{\mycontour{25.0}} & \mycontour{(6,5)} & \multicolumn{1}{c|}{-} & \multicolumn{1}{c|}{-} & \multicolumn{1}{c|}{\mycontour{25.2}} & \multicolumn{1}{c|}{\mycontour{(8,5)}} & \multicolumn{1}{c|}{\mycontour{25.2}} &  \mycontour{(8,5)} \\ \hline
15 & \multicolumn{1}{c|}{-} & \multicolumn{1}{c|}{-} & \multicolumn{1}{c|}{24.9} & \multicolumn{1}{c|}{(4,5)} & \multicolumn{1}{c|}{\mycontour{25.0}} & \mycontour{(6,5)} & \multicolumn{1}{c|}{-} & \multicolumn{1}{c|}{-} & \multicolumn{1}{c|}{\mycontour{25.3}} & \multicolumn{1}{c|}{\mycontour{(8,6)}} & \multicolumn{1}{c|}{\mycontour{25.3}} & \mycontour{(6,6)}  \\ \hline
16 & \multicolumn{1}{c|}{24.5} & \multicolumn{1}{c|}{(3,3)} & \multicolumn{1}{c|}{25.0} & \multicolumn{1}{c|}{(6,6)} & \multicolumn{1}{c|}{\mycontour{25.1}} & \mycontour{(6,7)} & \multicolumn{1}{c|}{24.6} & \multicolumn{1}{c|}{(3,3)} & \multicolumn{1}{c|}{\mycontour{25.4}} & \multicolumn{1}{c|}{\mycontour{(8,6)}} & \multicolumn{1}{c|}{\mycontour{25.4}} & \mycontour{(6,6)}  \\ \hline
17 & \multicolumn{1}{c|}{24.7} & \multicolumn{1}{c|}{(3,4)} & \multicolumn{1}{c|}{25.0} & \multicolumn{1}{c|}{(6,6)} & \multicolumn{1}{c|}{\mycontour{25.2}} & \mycontour{(8,7)} & \multicolumn{1}{c|}{25.0} & \multicolumn{1}{c|}{(3,4)} & \multicolumn{1}{c|}{\mycontour{25.5}} & \multicolumn{1}{c|}{\mycontour{(4,8)}} & \multicolumn{1}{c|}{\mycontour{25.5}} & \mycontour{(6,8)} \\ \hline
18 & \multicolumn{1}{c|}{24.8} & \multicolumn{1}{c|}{(3,5)} & \multicolumn{1}{c|}{25.0} & \multicolumn{1}{c|}{(6,7)} & \multicolumn{1}{c|}{\mycontour{25.2}} & \mycontour{(8,7)} & \multicolumn{1}{c|}{25.2} & \multicolumn{1}{c|}{(3,5)} & \multicolumn{1}{c|}{\mycontour{25.5}} & \multicolumn{1}{c|}{\mycontour{(4,8)}} & \multicolumn{1}{c|}{\mycontour{25.5}} & \mycontour{(6,8)} \\ \hline
\end{tabular}}
\label{tbl:pareto_bsd300}
\end{table*}

\subsection{Optimizing for Accumulator Constraints}
\label{sec:per_acc_dt}

\begin{figure*}[t!]
\centering
\subfloat[A2Q]{
    \includegraphics[width=0.96\linewidth]{
        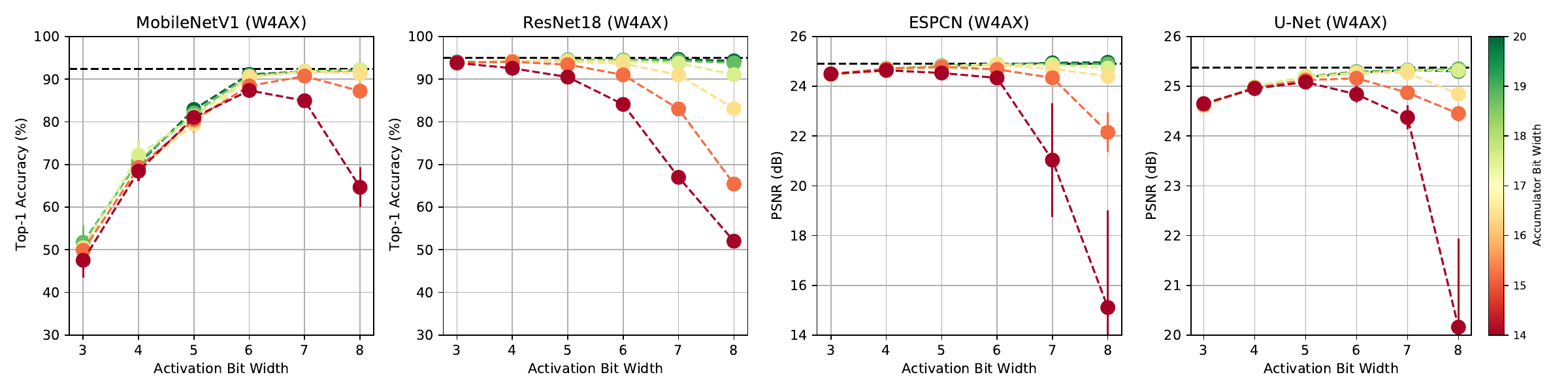
    }
} \\
\subfloat[A2Q+]{
    \includegraphics[width=0.96\linewidth]{
        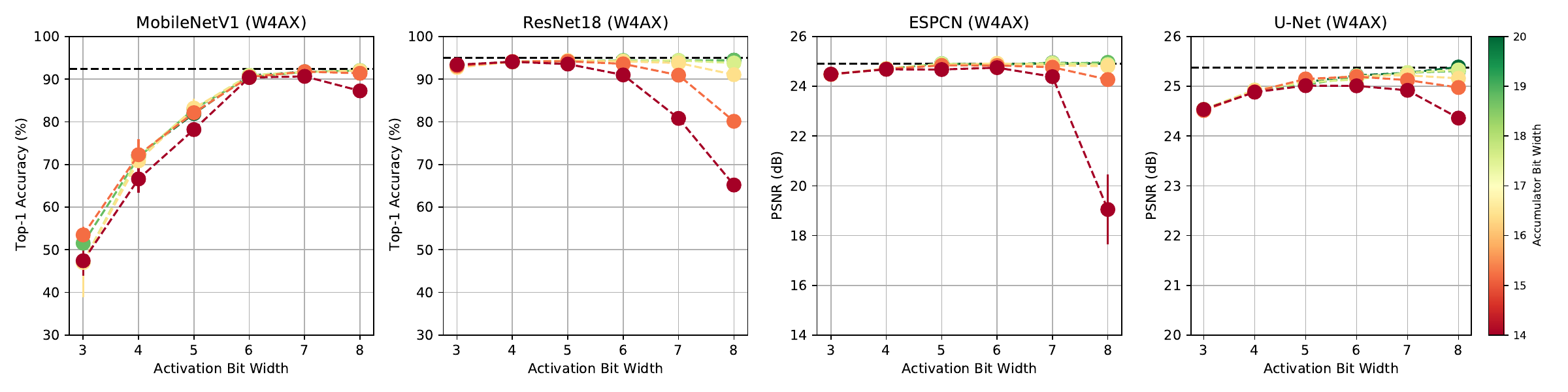
    }
}
\caption{We evaluate the trade-off between activation bit width~$N$ and model accuracy under fixed accumulator constraints.
We visualize the average and standard deviation in model accuracy measured over 3 experiments as~$N$ is increased from $3$ to $8$ bits when targeting accumulator widths that range from $14$ to $20$ bits.
The weights of all hidden layers are fixed to $4$-bits.}
\label{fig:w4ax}
\end{figure*}

{Following the benchmarking strategy in~\citealt{colbert2023a2q}, we first optimize QNNs for accumulator-constrained processors.
Here, the goal is to maximize model accuracy\footnote{We loosely use the term \textit{model accuracy} to also describe peak signal-to-noise ratio (PSNR) for convenience of discussion.} given a target accumulator bit width $P$.
This scenario has implications for accelerating inference on general-purpose platforms~\cite{xie2021overflow, li2022downscaling} and reducing the computational overhead of encrypted computations~\cite{lou2019she, stoian2023deep}.}
{As an alternative to A2Q, one could also heuristically manipulating weight bit width~$M$ and activation bit width~$N$ according to Eq.~\ref{eq:data_type_bound}.}
To the best of our knowledge, this is the only other method to train a uniform-precision QNN for a given $P$ without overflow.
Therefore, we use exhaustive bit width manipulation as a baseline when comparing A2Q+ against A2Q.

{In Fig.~\ref{fig:per_acc_dt}, we visualize this comparison using Pareto frontiers and provide the 32-bit floating-point model accuracy for reference.
For each model and each QAT algorithm, the Pareto frontier provides the maximum observed model accuracy for a given target accumulator bit width $P$.
In addition, we provide a detailed breakdown of each Pareto frontier in Tables~\ref{tbl:pareto_cifar10} and~\ref{tbl:pareto_bsd300}, where we also report the weight and activation bit widths of the Pareto-dominant model.
In these experiments, all super resolution benchmarks are trained from scratch and all image classification benchmarks are initialized from pre-trained floating-point checkpoints using our Euclidean projection-based weight initialization (EP-init).
We handle depthwise separable convolutions using the technique discussed in Appendix~\ref{appendix:dws_conv_ablation}, which only impacts MobileNetV1.
It is important to note that this is not a direct comparison against ~\citealt{colbert2023a2q} because we apply EP-init to both A2Q and A2Q+ models to strictly compare weight quantizers.
However, we  provide an ablation study in Appendix~\ref{appendix:l1_proj_ablation} that shows EP-init improves both A2Q and A2Q+ by up to $+50\%$ in extremely low-precision accumulation regimes.}

{Intuitively, heuristic bit width manipulations can only reduce the accumulator bit width so far because $P$ is ultimately limited by dot product size $K$.
Alternatively, using A2Q to train QNNs directly for low-precision accumulation allows one to push the accumulator bit width lower than previously attainable without compromising overflow; yet, a trade-off still remains.
We observe that A2Q+ significantly improves this trade-off, especially in the extremely low-precision accumulation regime.
Thus, by alleviating the pressure on model weights, A2Q+ recovers model accuracy lost to the overly restrictive accumulator constraints imposed by A2Q.

Finally, we observe that the Pareto-optimal activation bit width~$N$ monotonically decreases along the frontier as the target accumulator bit width~$P$ is reduced.
We hypothesize this is in part a consequence of the alleviated pressure on~$\Vert \vec{q} \Vert_1$ discussed in Section~\ref{sec:a2q}.
To investigate this relationship, we evaluate model accuracy as we increase~$N$ with fixed $P$.
To focus on $N$ and $P$, we fix weights to $4$~bits.
Figure~\ref{fig:w4ax} shows the average accuracy for each model as we increase $N$ from $3$ to $8$ bits when targeting $14$- to $20$-bit accumulation.
As previously established, reducing $P$ invariably limits model accuracy.
However, additionally increasing $N$ continues to tighten the constraint on $\Vert \vec{q} \Vert_1$, further limiting model accuracy and introducing a non-trivial trade-off observed across neural architectures.
When compared to A2Q, A2Q+ significantly alleviates this trade-off by alleviating the constraints on~$\Vert \vec{q} \Vert_1$ for fixed $N$ and $P$, increasing the Pareto-optimal activation bit width and model accuracy.

\subsection{Low-Precision Accumulation for ImageNet Models}
\label{sec:more_models}

The $\ell_1$-norm of an unconstrained weight vector inherently grows as its dimensionality~$K$ increases.
This suggests that, with a fixed activation bit width~$N$ and target accumulator bit width~$P$, A2Q and A2Q+ scale well to deeper architectures as the accumulator constraint tightens with the width of a neural architecture rather than the depth.}

We investigate this hypothesis by evaluating larger ResNet models trained on ImageNet from pre-trained floating-point checkpoints.
Rather than exploring the full quantization design space, we focus on 4-bit weights and activations while evaluating A2Q and A2Q+ under various accumulator constraints.
We also evaluate the impact of our Euclidean projection-based weight initialization strategy (EP-init) on standard A2Q and use the standard methods discussed in Appendix~\ref{appendix:hyperparams} to provide a reference QAT baseline.
We use the pre-trained checkpoints provided by PyTorch~\cite{paszke2019pytorch} and report our results in Table~\ref{tbl:imagenet}.

{We observe that both A2Q and A2Q+ can maintain baseline accuracy when targeting $16$-bit accumulators; however, it is important to note that this is a non-trivial result.
Only about $50\%$ of the output channels in the PyTorch ResNet18 and ResNet34 checkpoints inherently satisfy a $16$-bit accumulator constraint, with less than $10\%$ satisfying 12-bit constraints (see Appendix~\ref{appendix:l1_proj_ablation}).
While A2Q is able to recover when $P=16$, we observe that EP-init significantly improves model accuracy as $P$ is reduced, with a noteable $+11.7\%$ increase in test top-1 accuracy on ResNet50 when targeting $12$-bit accumulation.
Furthermore, we observe that A2Q+ can consistently maintain over $96\%$ of the test top-1 accuracy relative to the $32$-bit floating-point baselines when targeting $14$-bit accumulators.}

Interestingly, we observe that the accuracy gap between accumulator-constrained models and their original floating-point counterparts decreases as model size increases.
Building from our hypothesis, we conjecture this is in part because the models are growing in depth but not width, which increases model capacity without tightening our constraints.

{Finally, we observe that both A2Q and A2Q+ inherently expose opportunities to exploit unstructured weight sparsity.
As shown in~\citealt{colbert2023a2q}, decreasing $P$ increases sparsity when $N$ is fixed.
Furthermore, since A2Q is more restrictive than A2Q+ for fixed $P$ and $N$, we see that A2Q can result in significantly higher sparsity levels.
We observe this gap in sparsity decreases with $P$ while the accuracy gap increases, with A2Q+ resulting in $+17\%$ top-1 accuracy with only $-6.3\%$ sparsity when compared to A2Q for $12$-bit accumulator constraints on ResNet50.

\begin{table}[b!]
\centering
\vspace{-0.3cm}
\caption{We evaluate A2Q+ for W4A4 ImageNet models  and compare against baseline QAT methods and standard A2Q, both with and without Euclidean projection-based initialization (EP-init.)}
\begin{tabular}{|c|c|c|P{0.12\linewidth}|P{0.13\linewidth}|}
\hline
{\textbf{Network}} & \textbf{Method} & {$\mathbf{P}$} & \textbf{Top-1} & \textbf{Sparsity} \\ \hline \hline
  \multirow{10}{*}{\thead{\textbf{ResNet18} \\ (Float: 69.76\%)}} & Base & 32 & 70.2\% & 20.8\%  \\ \cline{2-5} 
                  & \multirow{3}{*}{\thead{A2Q \\ }} & 16 & 69.2\% & 73.7\% \\ \cline{3-5} 
                  &                   & 14 & 60.3\% & 91.7\% \\ \cline{3-5}
                  &                   & 12 & 35.5\% & 94.7\% \\ \cline{2-5}
                  & \multirow{3}{*}{\thead{A2Q \\ (w/ EP-init)}}  & 16 & 69.3\% & 73.7\% \\ \cline{3-5} 
                  &                   & 14 & 62.5\% & 91.2\% \\ \cline{3-5}
                  &                   & 12 & 42.7\% & 94.6\% \\ \cline{2-5}
                  & \multirow{3}{*}{\mycontour{A2Q+}} & \mycontour{16} & \mycontour{69.8\%} & \mycontour{50.4\%} \\ \cline{3-5} 
                  &                   & \mycontour{14} & \mycontour{67.1\%} & \mycontour{85.0\%} \\ \cline{3-5}
                  &                   & \mycontour{12} & \mycontour{56.4\%} & \mycontour{93.4\%} \\ \hline \hline

 \multirow{10}{*}{\thead{\textbf{ResNet34} \\ (Float: 73.31\%)}} & Base & 32 & 73.4\% & 23.9\% \\ \cline{2-5} 
                  & \multirow{3}{*}{\thead{A2Q \\ }} & 16 & 73.1\% & 75.2\% \\ \cline{3-5}
                  &                   & 14 & 65.8\% & 94.5\% \\ \cline{3-5}
                  &                   & 12 & 43.4\% & 96.9\% \\ \cline{2-5}
                  & \multirow{3}{*}{\thead{A2Q \\ (w/ EP-init)}} & 16 & 73.1\% & 74.9\% \\ \cline{3-5} 
                  &                   & 14 & 67.6\% & 94.2\% \\ \cline{3-5}
                  &                   & 12 & 51.4\% & 96.8\% \\ \cline{2-5}
                  & \multirow{3}{*}{\mycontour{A2Q+}} & \mycontour{16} & \mycontour{73.3\%} & \mycontour{51.4\%} \\ \cline{3-5}
                  &                   & \mycontour{14} & \mycontour{71.4\%} & \mycontour{85.0\%} \\ \cline{3-5}
                  &                   & \mycontour{12} & \mycontour{62.1\%} & \mycontour{95.9\%} \\ \hline \hline
 
 \multirow{10}{*}{\thead{\textbf{ResNet50} \\ (Float: 76.13\%)}} & Base & 32 & 75.9\% & 25.8\% \\ \cline{2-5} 
                  & \multirow{3}{*}{\thead{A2Q \\ }} & 16 & 76.0\% & 56.1\% \\ \cline{3-5} 
                  &                   & 14 & 73.8\% & 77.2\% \\ \cline{3-5} 
                  &                   & 12 & 55.0\% & 90.7\% \\ \cline{2-5}
                  & \multirow{3}{*}{\thead{A2Q \\ (w/ EP-init)}} & 16 & 76.0\% & 56.1\% \\ \cline{3-5} 
                  &                   & 14 & 74.5\% & 77.1\% \\ \cline{3-5} 
                  &                   & 12 & 66.7\% & 88.6\% \\ \cline{2-5} 
                  & \multirow{3}{*}{\mycontour{A2Q+}} & \mycontour{16} & \mycontour{76.0\%} & \mycontour{44.0\%} \\ \cline{3-5} 
                  &                   & \mycontour{14} & \mycontour{75.7\%} & \mycontour{67.7\%} \\ \cline{3-5} 
                  &                   & \mycontour{12} & \mycontour{72.0\%} & \mycontour{84.4\%}  \\ \hline
 
\end{tabular}
\label{tbl:imagenet}
\end{table}

\section{Conclusions and Future Work}
\label{sec:discussion}

As weights and activations are increasingly represented with fewer bits, we anticipate the accumulator to play a critical role in the quantization design space.
However, while reducing the precision of the accumulator offers significant hardware efficiency improvements, it also invariably limits model accuracy by means of either numerical overflow or learning constraints~\cite{de2020quantization, ni2021wrapnet, xie2021overflow, colbert2023a2q}.
Our results show that A2Q+ significantly improves this trade-off, outperforming prior methods that guarantee overflow avoidance.

A2Q+ uses zero-centering to alleviate the $\ell_1$-norm constraints of A2Q, improving model accuracy without compromising overflow avoidance.
It is important to note that prior work has also studied benefits of zero-centering in other contexts.
\citealt{huang2017centered} show that normalizing weights to have zero mean and unit $\ell_2$-norm can stabilize pre-activation distributions and yield better-conditioned optimization problems.
\citealt{qiao2019micro} show that normalizing weights to instead have zero mean and unit variance can smooth the loss landscape and improve convergence.
\citealt{li2019additive} propose a non-uniform quantization scheme that also normalizes weights to have zero mean and unit variance and report increased training stability.}
However, this collection of favorable properties may not directly translate to A2Q+, which is a uniform quantization scheme that normalizes each output channel of the weights to have zero mean and unit $\ell_1$-norm, but we do observe that A2Q+ inherits an unfavorable property of zero-centering that seems to have been overlooked in these prior works: implicit dimensionality reduction.
We observe this to only negatively impact on depthwise separable convolutions (see  Appendix~\ref{appendix:dws_conv_ablation}).

A2Q+ uses Euclidean projections to minimize weight quantization error at initialization.
As a consequence of accumulator constraints, na\"ive initialization forces models to recover from superfluous quantization error as $P$ is reduced.
However, while our experiments show that minimizing the weight quantization error at initialization yields significant improvements in the resulting model accuracy, we do not observe increased post-training quantization performance.
Similar to~\citealt{colbert2023a2q}, this is due to the reliance on round-to-zero and the severity of accumulator constraints, which we highlight for future work.

Finally, A2Q+ introduces unstructured weight sparsity as the accumulator bit width is reduced.
Although studies have exploited unstructured sparsity to improve inference performance on both programmable logic~\cite{nurvitadhi2017can, colbert2021competitive} and general-purpose platforms~\cite{elsen2020fast, gale2020sparse}, many off-the-shelf accelerators require structured patterns to see performance benefits~\cite{mao2017exploring, mishra2021accelerating}.
We highlight controlling weight sparsity patterns for future work.

\ifdefined\isaccepted
\section*{Acknowledgements}

We would like to thank Gabor Sines, Michaela Blott, Benoit Jacob, Giuseppe Franco, Jake Daly, and Mehdi Saeedi from AMD for their feedback and support.
We would also like to thank Rayan Saab and Jinjie Zhang from UC San Diego and Alec Flowers from EPFL for insightful discussions. \\

© 2024 Advanced Micro Devices, Inc.  All rights reserved.
AMD, the AMD Arrow logo, Radeon, and combinations thereof are trademarks of Advanced Micro Devices, Inc.
Other product names used in this publication are for identification purposes only and may be trademarks of their respective companies.
\fi

\bibliography{references}

\begin{thebibliography}{53}
\providecommand{\natexlab}[1]{#1}
\providecommand{\url}[1]{\texttt{#1}}
\expandafter\ifx\csname urlstyle\endcsname\relax
  \providecommand{\doi}[1]{doi: #1}\else
  \providecommand{\doi}{doi: \begingroup \urlstyle{rm}\Url}\fi

\bibitem[Aggarwal et~al.(2023)Aggarwal, Pappalardo, Damsgaard, Franco,
  Preu{\ss}er, Blott, and Mitra]{aggarwal2023post}
Aggarwal, S., Pappalardo, A., Damsgaard, H.~J., Franco, G., Preu{\ss}er, T.~B.,
  Blott, M., and Mitra, T.
\newblock Post-training quantization with low-precision minifloats and integers
  on {FPGAs}.
\newblock \emph{arXiv preprint arXiv:2311.12359}, 2023.

\bibitem[Azamat et~al.(2022)Azamat, Park, and Lee]{azamat2022squeezing}
Azamat, A., Park, J., and Lee, J.
\newblock Squeezing accumulators in binary neural networks for extremely
  resource-constrained applications.
\newblock In \emph{Proceedings of the 41st IEEE/ACM International Conference on
  Computer-Aided Design}, pp.\  1--7, 2022.

\bibitem[Baier et~al.(2019)Baier, J{\"o}hren, and
  Seebacher]{baier2019challenges}
Baier, L., J{\"o}hren, F., and Seebacher, S.
\newblock Challenges in the deployment and operation of machine learning in
  practice.
\newblock In \emph{ECIS}, volume~1, 2019.

\bibitem[Bengio et~al.(2013)Bengio, L{\'e}onard, and
  Courville]{bengio2013estimating}
Bengio, Y., L{\'e}onard, N., and Courville, A.
\newblock Estimating or propagating gradients through stochastic neurons for
  conditional computation.
\newblock \emph{arXiv preprint arXiv:1308.3432}, 2013.

\bibitem[Blott et~al.(2018)Blott, Preusser, Fraser, Gambardella, O’brien,
  Umuroglu, Leeser, and Vissers]{blott2018finn}
Blott, M., Preusser, T.~B., Fraser, N.~J., Gambardella, G., O’brien, K.,
  Umuroglu, Y., Leeser, M., and Vissers, K.
\newblock {FINN-R}: An end-to-end deep-learning framework for fast exploration
  of quantized neural networks.
\newblock \emph{ACM Transactions on Reconfigurable Technology and Systems
  (TRETS)}, 11\penalty0 (3):\penalty0 1--23, 2018.

\bibitem[Blumenfeld et~al.(2023)Blumenfeld, Hubara, and
  Soudry]{blumenfeld2023towards}
Blumenfeld, Y., Hubara, I., and Soudry, D.
\newblock Towards cheaper inference in deep networks with lower bit-width
  accumulators.
\newblock In \emph{Workshop on Advancing Neural Network Training: Computational
  Efficiency, Scalability, and Resource Optimization (WANT@ NeurIPS 2023)},
  2023.

\bibitem[Colbert et~al.(2021{\natexlab{a}})Colbert, Daly, Kreutz-Delgado, and
  Das]{colbert2021competitive}
Colbert, I., Daly, J., Kreutz-Delgado, K., and Das, S.
\newblock A competitive edge: Can {FPGAs} beat {GPUs} at {DCNN} inference
  acceleration in resource-limited edge computing applications?
\newblock \emph{arXiv preprint arXiv:2102.00294}, 2021{\natexlab{a}}.

\bibitem[Colbert et~al.(2021{\natexlab{b}})Colbert, Kreutz-Delgado, and
  Das]{colbert2021energy}
Colbert, I., Kreutz-Delgado, K., and Das, S.
\newblock An energy-efficient edge computing paradigm for convolution-based
  image upsampling.
\newblock \emph{IEEE Access}, 9:\penalty0 147967--147984, 2021{\natexlab{b}}.

\bibitem[Colbert et~al.(2023)Colbert, Pappalardo, and
  Petri-Koenig]{colbert2023a2q}
Colbert, I., Pappalardo, A., and Petri-Koenig, J.
\newblock {A2Q}: Accumulator-aware quantization with guaranteed overflow
  avoidance.
\newblock In \emph{Proceedings of the IEEE/CVF International Conference on
  Computer Vision}, pp.\  16989--16998, 2023.

\bibitem[de~Bruin et~al.(2020)de~Bruin, Zivkovic, and
  Corporaal]{de2020quantization}
de~Bruin, B., Zivkovic, Z., and Corporaal, H.
\newblock Quantization of deep neural networks for accumulator-constrained
  processors.
\newblock \emph{Microprocessors and Microsystems}, 72:\penalty0 102872, 2020.

\bibitem[Deng et~al.(2009)Deng, Dong, Socher, Li, Li, and
  Fei-Fei]{deng2009imagenet}
Deng, J., Dong, W., Socher, R., Li, L.-J., Li, K., and Fei-Fei, L.
\newblock Imagenet: A large-scale hierarchical image database.
\newblock In \emph{2009 IEEE conference on computer vision and pattern
  recognition}, pp.\  248--255. Ieee, 2009.

\bibitem[Duchi et~al.(2008)Duchi, Shalev-Shwartz, Singer, and
  Chandra]{duchi2008efficient}
Duchi, J., Shalev-Shwartz, S., Singer, Y., and Chandra, T.
\newblock Efficient projections onto the l1-ball for learning in high
  dimensions.
\newblock In \emph{Proceedings of the 25th international conference on Machine
  learning}, pp.\  272--279, 2008.

\bibitem[Elsen et~al.(2020)Elsen, Dukhan, Gale, and Simonyan]{elsen2020fast}
Elsen, E., Dukhan, M., Gale, T., and Simonyan, K.
\newblock Fast sparse convnets.
\newblock In \emph{Proceedings of the IEEE/CVF conference on computer vision
  and pattern recognition}, pp.\  14629--14638, 2020.

\bibitem[Gale et~al.(2020)Gale, Zaharia, Young, and Elsen]{gale2020sparse}
Gale, T., Zaharia, M., Young, C., and Elsen, E.
\newblock Sparse {GPU} kernels for deep learning.
\newblock In \emph{SC20: International Conference for High Performance
  Computing, Networking, Storage and Analysis}, pp.\  1--14. IEEE, 2020.

\bibitem[Gholami et~al.(2021)Gholami, Kim, Dong, Yao, Mahoney, and
  Keutzer]{gholami2021survey}
Gholami, A., Kim, S., Dong, Z., Yao, Z., Mahoney, M.~W., and Keutzer, K.
\newblock A survey of quantization methods for efficient neural network
  inference.
\newblock \emph{arXiv preprint arXiv:2103.13630}, 2021.

\bibitem[Hardy et~al.(1952)Hardy, Littlewood, and
  P{\'o}lya]{hardy1952inequalities}
Hardy, G.~H., Littlewood, J.~E., and P{\'o}lya, G.
\newblock \emph{Inequalities}.
\newblock Cambridge university press, 1952.

\bibitem[Hawks et~al.(2021)Hawks, Duarte, Fraser, Pappalardo, Tran, and
  Umuroglu]{hawks2021ps}
Hawks, B., Duarte, J., Fraser, N.~J., Pappalardo, A., Tran, N., and Umuroglu,
  Y.
\newblock {Ps and Qs}: Quantization-aware pruning for efficient low latency
  neural network inference.
\newblock \emph{Frontiers in Artificial Intelligence}, 4:\penalty0 676564,
  2021.

\bibitem[He et~al.(2015)He, Zhang, Ren, and Sun]{he2015delving}
He, K., Zhang, X., Ren, S., and Sun, J.
\newblock Delving deep into rectifiers: Surpassing human-level performance on
  imagenet classification.
\newblock In \emph{Proceedings of the IEEE international conference on computer
  vision}, pp.\  1026--1034, 2015.

\bibitem[He et~al.(2016)He, Zhang, Ren, and Sun]{he2016deep}
He, K., Zhang, X., Ren, S., and Sun, J.
\newblock Deep residual learning for image recognition.
\newblock In \emph{Proceedings of the IEEE conference on computer vision and
  pattern recognition}, pp.\  770--778, 2016.

\bibitem[Horowitz(2014)]{horowitz20141}
Horowitz, M.
\newblock 1.1 computing's energy problem (and what we can do about it).
\newblock In \emph{2014 IEEE international solid-state circuits conference
  digest of technical papers (ISSCC)}, pp.\  10--14. IEEE, 2014.

\bibitem[Howard et~al.(2017)Howard, Zhu, Chen, Kalenichenko, Wang, Weyand,
  Andreetto, and Adam]{howard2017mobilenets}
Howard, A.~G., Zhu, M., Chen, B., Kalenichenko, D., Wang, W., Weyand, T.,
  Andreetto, M., and Adam, H.
\newblock Mobilenets: Efficient convolutional neural networks for mobile vision
  applications.
\newblock \emph{arXiv preprint arXiv:1704.04861}, 2017.

\bibitem[Huang et~al.(2017)Huang, Liu, Liu, Lang, and Tao]{huang2017centered}
Huang, L., Liu, X., Liu, Y., Lang, B., and Tao, D.
\newblock Centered weight normalization in accelerating training of deep neural
  networks.
\newblock In \emph{Proceedings of the IEEE International Conference on Computer
  Vision}, pp.\  2803--2811, 2017.

\bibitem[Jacob et~al.(2018)Jacob, Kligys, Chen, Zhu, Tang, Howard, Adam, and
  Kalenichenko]{jacob2018quantization}
Jacob, B., Kligys, S., Chen, B., Zhu, M., Tang, M., Howard, A., Adam, H., and
  Kalenichenko, D.
\newblock Quantization and training of neural networks for efficient
  integer-arithmetic-only inference.
\newblock In \emph{Proceedings of the IEEE conference on computer vision and
  pattern recognition}, pp.\  2704--2713, 2018.

\bibitem[Jain et~al.(2020)Jain, Gural, Wu, and Dick]{jain2020trained}
Jain, S., Gural, A., Wu, M., and Dick, C.
\newblock Trained quantization thresholds for accurate and efficient
  fixed-point inference of deep neural networks.
\newblock \emph{Proceedings of Machine Learning and Systems}, 2:\penalty0
  112--128, 2020.

\bibitem[Kingma \& Ba(2014)Kingma and Ba]{kingma2014adam}
Kingma, D.~P. and Ba, J.
\newblock Adam: A method for stochastic optimization.
\newblock \emph{arXiv preprint arXiv:1412.6980}, 2014.

\bibitem[Krizhevsky et~al.(2009)Krizhevsky, Hinton,
  et~al.]{krizhevsky2009learning}
Krizhevsky, A., Hinton, G., et~al.
\newblock Learning multiple layers of features from tiny images.
\newblock 2009.

\bibitem[Li et~al.(2022)Li, Liu, Jia, Liang, Wang, and Tan]{li2022downscaling}
Li, H., Liu, J., Jia, L., Liang, Y., Wang, Y., and Tan, M.
\newblock Downscaling and overflow-aware model compression for efficient vision
  processors.
\newblock In \emph{2022 IEEE 42nd International Conference on Distributed
  Computing Systems Workshops (ICDCSW)}, pp.\  145--150. IEEE, 2022.

\bibitem[Li et~al.(2019)Li, Dong, and Wang]{li2019additive}
Li, Y., Dong, X., and Wang, W.
\newblock Additive powers-of-two quantization: An efficient non-uniform
  discretization for neural networks.
\newblock \emph{arXiv preprint arXiv:1909.13144}, 2019.

\bibitem[Loroch et~al.(2017)Loroch, Pfreundt, Wehn, and
  Keuper]{loroch2017tensorquant}
Loroch, D.~M., Pfreundt, F.-J., Wehn, N., and Keuper, J.
\newblock Tensorquant: A simulation toolbox for deep neural network
  quantization.
\newblock In \emph{Proceedings of the Machine Learning on HPC Environments},
  pp.\  1--8. 2017.

\bibitem[Lou \& Jiang(2019)Lou and Jiang]{lou2019she}
Lou, Q. and Jiang, L.
\newblock She: A fast and accurate deep neural network for encrypted data.
\newblock \emph{Advances in Neural Information Processing Systems}, 32, 2019.

\bibitem[Mao et~al.(2017)Mao, Han, Pool, Li, Liu, Wang, and
  Dally]{mao2017exploring}
Mao, H., Han, S., Pool, J., Li, W., Liu, X., Wang, Y., and Dally, W.~J.
\newblock Exploring the granularity of sparsity in convolutional neural
  networks.
\newblock In \emph{Proceedings of the IEEE Conference on Computer Vision and
  Pattern Recognition Workshops}, pp.\  13--20, 2017.

\bibitem[Martin et~al.(2001)Martin, Fowlkes, Tal, and
  Malik]{martin2001database}
Martin, D., Fowlkes, C., Tal, D., and Malik, J.
\newblock A database of human segmented natural images and its application to
  evaluating segmentation algorithms and measuring ecological statistics.
\newblock In \emph{Proceedings Eighth IEEE International Conference on Computer
  Vision. ICCV 2001}, volume~2, pp.\  416--423. IEEE, 2001.

\bibitem[Mishra et~al.(2021)Mishra, Latorre, Pool, Stosic, Stosic, Venkatesh,
  Yu, and Micikevicius]{mishra2021accelerating}
Mishra, A., Latorre, J.~A., Pool, J., Stosic, D., Stosic, D., Venkatesh, G.,
  Yu, C., and Micikevicius, P.
\newblock Accelerating sparse deep neural networks.
\newblock \emph{arXiv preprint arXiv:2104.08378}, 2021.

\bibitem[Nagel et~al.(2019)Nagel, Baalen, Blankevoort, and
  Welling]{nagel2019data}
Nagel, M., Baalen, M.~v., Blankevoort, T., and Welling, M.
\newblock Data-free quantization through weight equalization and bias
  correction.
\newblock In \emph{Proceedings of the IEEE/CVF International Conference on
  Computer Vision}, pp.\  1325--1334, 2019.

\bibitem[Ni et~al.(2021)Ni, Chu, Casta{\~n}eda~Fern{\'a}ndez, Chiang, Studer,
  and Goldstein]{ni2021wrapnet}
Ni, R., Chu, H.-m., Casta{\~n}eda~Fern{\'a}ndez, O., Chiang, P.-y., Studer, C.,
  and Goldstein, T.
\newblock Wrapnet: Neural net inference with ultra-low-precision arithmetic.
\newblock In \emph{International Conference on Learning Representations ICLR
  2021}. OpenReview, 2021.

\bibitem[Nurvitadhi et~al.(2017)Nurvitadhi, Venkatesh, Sim, Marr, Huang, Ong
  Gee~Hock, Liew, Srivatsan, Moss, Subhaschandra, et~al.]{nurvitadhi2017can}
Nurvitadhi, E., Venkatesh, G., Sim, J., Marr, D., Huang, R., Ong Gee~Hock, J.,
  Liew, Y.~T., Srivatsan, K., Moss, D., Subhaschandra, S., et~al.
\newblock Can {FPGAs} beat {GPUs} in accelerating next-generation deep neural
  networks?
\newblock In \emph{Proceedings of the 2017 ACM/SIGDA international symposium on
  field-programmable gate arrays}, pp.\  5--14, 2017.

\bibitem[Odena et~al.(2016)Odena, Dumoulin, and Olah]{odena2016deconvolution}
Odena, A., Dumoulin, V., and Olah, C.
\newblock Deconvolution and checkerboard artifacts.
\newblock \emph{Distill}, 1\penalty0 (10):\penalty0 e3, 2016.

\bibitem[Pappalardo(2021)]{brevitas}
Pappalardo, A.
\newblock Xilinx/brevitas, 2021.
\newblock URL \url{https://doi.org/10.5281/zenodo.3333552}.

\bibitem[Paszke et~al.(2019)Paszke, Gross, Massa, Lerer, Bradbury, Chanan,
  Killeen, Lin, Gimelshein, Antiga, et~al.]{paszke2019pytorch}
Paszke, A., Gross, S., Massa, F., Lerer, A., Bradbury, J., Chanan, G., Killeen,
  T., Lin, Z., Gimelshein, N., Antiga, L., et~al.
\newblock Pytorch: An imperative style, high-performance deep learning library.
\newblock \emph{Advances in neural information processing systems},
  32:\penalty0 8026--8037, 2019.

\bibitem[Qiao et~al.(2019)Qiao, Wang, Liu, Shen, and Yuille]{qiao2019micro}
Qiao, S., Wang, H., Liu, C., Shen, W., and Yuille, A.
\newblock Micro-batch training with batch-channel normalization and weight
  standardization.
\newblock \emph{arXiv preprint arXiv:1903.10520}, 2019.

\bibitem[Ronneberger et~al.(2015)Ronneberger, Fischer, and
  Brox]{ronneberger2015u}
Ronneberger, O., Fischer, P., and Brox, T.
\newblock U-net: Convolutional networks for biomedical image segmentation.
\newblock In \emph{International Conference on Medical image computing and
  computer-assisted intervention}, pp.\  234--241. Springer, 2015.

\bibitem[Salimans \& Kingma(2016)Salimans and Kingma]{salimans2016weight}
Salimans, T. and Kingma, D.~P.
\newblock Weight normalization: A simple reparameterization to accelerate
  training of deep neural networks.
\newblock \emph{Advances in neural information processing systems}, 29, 2016.

\bibitem[Shi et~al.(2016)Shi, Caballero, Husz{\'a}r, Totz, Aitken, Bishop,
  Rueckert, and Wang]{shi2016real}
Shi, W., Caballero, J., Husz{\'a}r, F., Totz, J., Aitken, A.~P., Bishop, R.,
  Rueckert, D., and Wang, Z.
\newblock Real-time single image and video super-resolution using an efficient
  sub-pixel convolutional neural network.
\newblock In \emph{Proceedings of the IEEE conference on computer vision and
  pattern recognition}, pp.\  1874--1883, 2016.

\bibitem[Sifre \& Mallat(2014)Sifre and Mallat]{sifre2014rigid}
Sifre, L. and Mallat, S.
\newblock Rigid-motion scattering for texture classification.
\newblock \emph{arXiv preprint arXiv:1403.1687}, 2014.

\bibitem[Stoian et~al.(2023)Stoian, Frery, Bredehoft, Montero, Kherfallah, and
  Chevallier-Mames]{stoian2023deep}
Stoian, A., Frery, J., Bredehoft, R., Montero, L., Kherfallah, C., and
  Chevallier-Mames, B.
\newblock Deep neural networks for encrypted inference with {TFHE}.
\newblock \emph{arXiv preprint arXiv:2302.10906}, 2023.

\bibitem[Umuroglu \& Jahre(2017)Umuroglu and Jahre]{umuroglu2017streamlined}
Umuroglu, Y. and Jahre, M.
\newblock Streamlined deployment for quantized neural networks.
\newblock \emph{arXiv preprint arXiv:1709.04060}, 2017.

\bibitem[Umuroglu et~al.(2017)Umuroglu, Fraser, Gambardella, Blott, Leong,
  Jahre, and Vissers]{finn}
Umuroglu, Y., Fraser, N.~J., Gambardella, G., Blott, M., Leong, P., Jahre, M.,
  and Vissers, K.
\newblock {FINN}: A framework for fast, scalable binarized neural network
  inference.
\newblock In \emph{Proceedings of the 2017 ACM/SIGDA International Symposium on
  Field-Programmable Gate Arrays}, FPGA '17, pp.\  65--74. ACM, 2017.

\bibitem[Wu et~al.(2020)Wu, Judd, Zhang, Isaev, and
  Micikevicius]{wu2020integer}
Wu, H., Judd, P., Zhang, X., Isaev, M., and Micikevicius, P.
\newblock Integer quantization for deep learning inference: Principles and
  empirical evaluation.
\newblock \emph{arXiv preprint arXiv:2004.09602}, 2020.

\bibitem[Wu et~al.(2023)Wu, Li, Aminabadi, Yao, and He]{wu2023understanding}
Wu, X., Li, C., Aminabadi, R.~Y., Yao, Z., and He, Y.
\newblock Understanding int4 quantization for language models: latency speedup,
  composability, and failure cases.
\newblock In \emph{International Conference on Machine Learning}, pp.\
  37524--37539. PMLR, 2023.

\bibitem[Xie et~al.(2021)Xie, Song, Cai, and Li]{xie2021overflow}
Xie, H., Song, Y., Cai, L., and Li, M.
\newblock Overflow aware quantization: Accelerating neural network inference by
  low-bit multiply-accumulate operations.
\newblock In \emph{Proceedings of the Twenty-Ninth International Conference on
  International Joint Conferences on Artificial Intelligence}, pp.\  868--875,
  2021.

\bibitem[Yang et~al.(2019)Yang, Pennington, Rao, Sohl-Dickstein, and
  Schoenholz]{yang2019mean}
Yang, G., Pennington, J., Rao, V., Sohl-Dickstein, J., and Schoenholz, S.~S.
\newblock A mean field theory of batch normalization.
\newblock \emph{arXiv preprint arXiv:1902.08129}, 2019.

\bibitem[Yao et~al.(2021)Yao, Dong, Zheng, Gholami, Yu, Tan, Wang, Huang, Wang,
  Mahoney, et~al.]{yao2021hawq}
Yao, Z., Dong, Z., Zheng, Z., Gholami, A., Yu, J., Tan, E., Wang, L., Huang,
  Q., Wang, Y., Mahoney, M., et~al.
\newblock {HAWQ-V3}: Dyadic neural network quantization.
\newblock In \emph{International Conference on Machine Learning}, pp.\
  11875--11886. PMLR, 2021.

\bibitem[Zhang et~al.(2022)Zhang, Colbert, and Das]{zhang2022learning}
Zhang, X., Colbert, I., and Das, S.
\newblock Learning low-precision structured subnetworks using joint layerwise
  channel pruning and uniform quantization.
\newblock \emph{Applied Sciences}, 12\penalty0 (15):\penalty0 7829, 2022.

\end{thebibliography}
\bibliographystyle{icml2024}

\clearpage
\appendix

\section{Proofs}

\subsection{Proof of Proposition~\ref{prop:upper_bound_on_l1_norm}}
\label{appendix:prop_1_proof}

Let weights $\vec{q}$ be a $K$-dimensional vector of $M$-bit integers, and let
$\mathbb{Z}^K_N$ denote the set of all $K$-dimensional vectors of $N$-bit integers.
To prove Prop.~\ref{prop:upper_bound_on_l1_norm}, restated below for completeness, we examine the vectors that maximize and minimize the dot product of $\vec{q}$ by any $\vec{x} \in \mathbb{Z}^K_N$ and directly derive our result by exhaustively evaluating each case.
Without loss of generality, we assume a two's complement representation for signed integers in our work as is common practice~\cite{wu2020integer, gholami2021survey}. \\

\propOne*

\begin{proof}
Let $\alpha$ denote the sum of all positive elements of~$\vec{q}$ and let $\beta$ denote the sum of all negative elements of $\vec{q}$.
It follows that $\alpha + \beta = 0$ (property of the zero-centered vector) and $\alpha - \beta = \oneNorm{q}$ (property of the $\ell_1$-norm).
This yields the following relationships: $\alpha = -\beta = \frac{1}{2} \oneNorm{\vec{q}}$.

Let the closed interval $[e , f]$ denote the output range of the dot product of $\vec{q}$ by any $\vec{x} \in \mathbb{Z}^K_N$, where $f \geq e$.
To safely use a signed $P$-bit accumulator without overflow, all of the following inequalities need to be satisfied:
\begin{align}
f \leq &~ 2^{P - 1} - 1 \label{eq:ineq_1} \\
-e \leq &~ 2^{P - 1} \label{eq:ineq_2} \\
f - e \leq &~ 2^P - 1 \label{eq:ineq_3}
\end{align}
We start with the first inequality, Eq.~\ref{eq:ineq_1}.
Since the value of each input element $x_i$ is bounded to the closed interval $[c,d]$, the maximizing vector $\vec{\mu}$ is defined as:
\begin{equation}
\vec{\mu} = \argmax_{\vec{x}} ~\vec{x}^T \vec{q} = \begin{cases}
d, & \text{where } q_i \geq 0 \\
c, & \text{where } q_i < 0
\end{cases}
\label{eq:mu_values}
\end{equation}
Exploiting the identities of $\alpha$, $\beta$, $c$, $d$, and $f$, we can derive the following upper bound on the $\ell_1$-norm of $\vec{q}$:
\begin{align}
\vec{\mu}^T \vec{q} & \leq 2^{P - 1} - 1 \\
d \alpha + c \beta & \leq 2^{P - 1} - 1 \\
\alpha (d - c) & \leq 2^{P - 1} - 1 \\
\oneNorm{\vec{q}} & \leq \frac{2^P - 2}{2^N - 1}
\label{eq:derivation_f}
\end{align}
Note that this aligns with Prop.~\ref{prop:upper_bound_on_l1_norm}.
Next, we prove that satisfying this bound will also satisfy Eqs.~\ref{eq:ineq_2} and~\ref{eq:ineq_3}.

We continue onto the second inequality, Eq.~\ref{eq:ineq_2}.
Similar to Eq.\ref{eq:mu_values}, the minimizing vector $\vec{\nu}$ is defined as:
\begin{equation}
\vec{\nu} = \argmin_{\vec{x}} ~~\vec{x}^T \vec{q} = \begin{cases}
c, & \text{where } q_i \geq 0 \\
d, & \text{where } q_i < 0
\end{cases}
\label{eq:nu_values}
\end{equation}
Again exploiting our defined identities, we can derive the following upper bound on the $\ell_1$-norm of $\vec{q}$:
\begin{align}
-\vec{\nu}^T \vec{q} & \leq 2^{P - 1} \\
-c \alpha - d \beta & \leq 2^{P - 1} \\
\alpha (d - c) & \leq 2^{P - 1} \\
\oneNorm{\vec{q}} & \leq \frac{2^P}{2^N - 1}
\label{eq:derivation_e}
\end{align}
Note that by satisfying Eq.~\ref{eq:derivation_f}, we also satisfy Eq.~\ref{eq:derivation_e}.

Finally, we evaluate the last inequality, Eq.~\ref{eq:ineq_3}.
From Eqs.~\ref{eq:mu_values} and~\ref{eq:nu_values}, it follows that $\mu_i - \nu_i = (d - c) \text{sign}(q_i)$.
With this new identity, we can derive the following upper bound on the $\ell_1$-norm of $\vec{q}$:
\begin{align}
\left(\vec{\mu} - \vec{\nu}\right)^T \vec{q} & \leq 2^P - 1 \\
(d - c) \text{sign}(\vec{q})^T \vec{q} & \leq 2^P - 1 \\
\oneNorm{\vec{q}} & \leq \frac{2^P - 1}{2^N - 1}
\label{eq:derivation_fe}
\end{align}
Thus, by satisfying Eq.~\ref{eq:derivation_f}, we also satisfy both Eqs.~\ref{eq:derivation_e} and~\ref{eq:derivation_fe}, enabling the use of a signed $P$-bit accumulator.
\end{proof}

\subsection{Proof of Proposition~\ref{prop:n2}}
\label{appendix:prop_2_proof}

To prove Prop.~\ref{prop:n2}, we first present the following lemma: \\

\begin{lemma}
{Let $\vec{x}$, $\vec{w}$, and $\vec{q}$ each be $K$-dimensional vectors.
If $\textrm{\emph{sign}}(x_i) = \textrm{\emph{sign}}(w_i) = \textrm{\emph{sign}}(q_i)$ for all non-zero $x_i$ and $\vert q_i \vert \leq \vert w_i \vert$ for all $i$, then $\vec{x}^T \vec{q} \leq \vec{x}^T \vec{w}$.}
\label{lem:my_lemma_1}
\end{lemma}
\begin{proof}
{Given that $\textrm{sign}(x_i) = \textrm{sign}(w_i) = \textrm{sign}(q_i)$ for all non-zero $x_i$, it follows that $\vec{x}^T \vec{q} = \textstyle \sum_i \vert x_i \vert \vert q_i \vert$ and $\vec{x}^T \vec{w} = \textstyle \sum_i \vert x_i \vert \vert w_i \vert$.
Using these identities, we can directly derive the following inequality:}
\begin{align}
\vec{x}^T \vec{q} &~\leq \vec{x}^T \vec{w} \\
\textstyle \sum_i \vert x_i \vert \vert q_i \vert &~\leq \textstyle \sum_i \vert x_i \vert \vert w_i \vert \\
\textstyle \sum_i \vert x_i \vert \left( \vert q_i \vert - \vert w_i \vert \right) &~\leq 0
\end{align}
{Given that $\vert q_i \vert \leq \vert w_i \vert$ for all $i$, this leads us to the desired result that the inequality holds, \textit{i.e.}, $\vec{x}^T \vec{q} \leq \vec{x}^T \vec{w}$.}
\end{proof}

Consider again inputs $\vec{x}$ and integer-quantized weights $\vec{q}$.
Recall that simulated {quantization} derives $\vec{q}$ from floating-point counterpart $\vec{w}$ using a transformation function referred to as a quantizer.
To prove Prop.~\ref{prop:n2}, we leverage the necessary accumulator constraint conditions formally articulated in Appendix~\ref{appendix:prop_1_proof}: Eqs.~\ref{eq:new_bound_prop_1}, ~\ref{eq:ineq_1}, ~\ref{eq:ineq_2}, and ~\ref{eq:ineq_3}.
We again directly derive our result by exhaustively evaluating each case. \\

\propTwo*

\begin{proof}
{As shown in Section~\ref{sec:derivation}, $\vec{q}$ must satisfy Eqs.~\ref{eq:new_bound_prop_1}, ~\ref{eq:ineq_1}, ~\ref{eq:ineq_2}, and ~\ref{eq:ineq_3} to avoid overflow when accumulating the result of $\vec{x}^T \vec{q}$ into a $P$-bit register.}
{To show that $\vec{q}$ satisfies these four necessary conditions when $\vec{w} / s$ does as well, we directly prove each case, starting with Eq.~\ref{eq:new_bound_prop_1}.
Given that $\vert s \cdot q_i \vert \leq \vert w_i \vert$ for all $i$ and $s$ is a strictly positive scalar, it follows that $\vert q_i \vert \leq \vert w_i /s \vert$ and thus $\Vert \vec{q} \Vert_1 \leq \Vert \vec{w} / s \Vert_1 $.
Therefore, when $\vec{w} / s$ satisfies Eq.~\ref{eq:new_bound_prop_1}, then $\vec{q}$ does as well.}

{To evaluate Eq.~\ref{eq:ineq_1}, let $\vec{\mu}$ be the vector that maximizes $\vec{x}^T \vec{w}$ as defined in Eq.~\ref{eq:mu_values}.
Given that $s$ is strictly positive and $\sign(s \cdot q_i) = \sign(w_i)$, it follows that $\sign(q_i) = \sign(w_i)$ and thus $\vec{\mu}$ also maximizes $\vec{x}^T \vec{q}$.
Furthermore, given that $\mu_i$ is an $N$-bit integer, the closed interval~$[c,d]$ is defined as~$[-2^{N - 1}, 2^{N - 1} - 1]$ when $\mu_i$ is signed and~$[0, 2^N - 1]$ when unsigned.
It follows that $\sign(\mu_i) = \sign(w_i) = \sign(q_i)$ for all non-zero~$\mu_i$ and $\abs{q_i} \leq \abs{w_i / s}$ for all $i$, and thus $\vec{\mu}^T \vec{q} \leq \vec{\mu}^T \vec{w} / s$ by Lemma~\ref{lem:my_lemma_1}.
Therefore, when $\vec{w} / s$ satisfies Eq.~\ref{eq:ineq_1}, then so does $\vec{q}$.}

{Similarly, let $\vec{\nu}$ be the vector that minimizes $\vec{x}^T \vec{w}$ as defined in Eq.~\ref{eq:nu_values}.
Given that $\sign(q_i) = \sign(w_i)$, then $\vec{\nu}$ also minimizes $\vec{x}^T \vec{q}$.
It again follows that $\sign(-\nu_i) - \sign({w_i}) = \sign(q_i)$ for all non-zero $\nu_i$ and $\abs{q_i} \leq \abs{w_i / s}$ for all $i$, and thus $-\vec{\nu}^T \vec{q} \leq - \vec{\nu}^T \vec{w} / s$ by Lemma~\ref{lem:my_lemma_1}.
Therefore, $\vec{q}$ satisfies Eq.~\ref{eq:ineq_2} when $\vec{w}/s$ does as well.}

{Following the same logic for Eq.~\ref{eq:ineq_3}, $\sign(\mu_i - \nu_i) = \sign(w_i) = \sign(q_i)$ for all $i$ where $\mu_i \neq \nu_i$, and thus $\left(\vec{\mu} - \vec{\nu}\right)^T \vec{q} \leq \left(\vec{\mu} - \vec{\nu}\right)^T \vec{w} / s$, again by Lemma~\ref{eq:ineq_3}.
Therefore, when $\vec{w} / s$ satisfies Eq.~\ref{eq:ineq_3}, then so does $\vec{q}$,  leading to the desired result for all four necessary conditions.}
\end{proof}

\section{Experiment Details \& Ablations}

\subsection{Hyperparameters \& Quantization Schemes}
\label{appendix:hyperparams}

Below, we provide further details on training hyperparameters, neural network architectures, and quantization schemes for our image classification and single-image super resolution benchmarks.
As we are building from the work of~\citealt{colbert2023a2q}, we adopt a quantization scheme that is amenable to compilation through FINN~\cite{finn}, where batch normalization layers, floating-point biases, and even scaling factors are absorbed into thresholding units via mathematical manipulation during graph compilation~\cite{umuroglu2017streamlined}.
Thus, we are not constrained to rely on power-of-2 scaling factors, quantized biases, or batch-norm folding as is common for integer-only inference~\cite{jacob2018quantization, wu2020integer, gholami2021survey}.
For all models, we fix the first and last layers to 8-bit weights and activations for all configurations, as is common practice~\cite{wu2020integer, gholami2021survey}.

{{Following the quantization scheme of~\citealt{colbert2023a2q}, we apply A2Q and A2Q+ to only the weights of a QNN and adopt the regularization penalty defined in Eq.~\ref{eq:reg_penalty} to avoid $g$ getting stuck when $g > T_+$.}
\begin{equation}
R = \max \{ g - T_+, 0 \}
\label{eq:reg_penalty}
\end{equation}
{This penalty is imposed on every hidden layer and combined into one regularizer: $\mathcal{L}_\text{reg} = \textstyle \sum_l \textstyle \sum_i R_{l, i}$, where $R_{l, i}$ denotes the regularization penalty for the $i$-th output channel in the $l$-th layer of the network.
We scale this regularization penalty $\mathcal{L}_{\textrm{reg}}$ by a constant scalar $\lambda=1e-3$ such that $\mathcal{L}_{\textrm{total}} = \mathcal{L}_{\textrm{task}} + \lambda \mathcal{L}_{\textrm{reg}}$, where $\mathcal{L}_{\textrm{task}}$ is the task-specific loss.

\textbf{Baseline QAT.} Our baseline QAT method is synthesized from common best practices.
Similar to A2Q and A2Q+, we symmetrically constrain the weight quantization scheme around the origin such that $z=0$ while allowing activations to be asymmetric~\cite{gholami2021survey, zhang2022learning}.
Eliminating these zero points on the weights reduces the computational overhead of cross-terms during integer-only inference~\cite{jacob2018quantization, jain2020trained}.
We use unique floating-point scaling factors for each output channel (or neuron) to adjust for varied dynamic ranges~\cite{nagel2019data}.
However, extending this strategy to activations can be computationally expensive~\cite{jain2020trained}.
As such, we use per-tensor scaling factors for activations and per-channel scaling factors on the weights, as is standard practice~\cite{jain2020trained, wu2020integer, zhang2022learning}.
Similar to A2Q and A2Q+, all scaling factors are learned in the log domain such that $s=2^d$, where $d$ is a log-scale learnable parameter.
The scaled weights (or activations) are rounded to the nearest integer and clipped to the limits of the representation range (see Section~\ref{sec:a2q_background}).
Noticeably, the rounding function introduces extremely sparse gradients; therefore, we use the straight-through estimator~\cite{bengio2013estimating} during training to allow gradients to permeate the rounding function such that $\nabla_x \lfloor x \rceil = 1$ everywhere and $\nabla_x$ denotes the gradient with respect to $x$.

\textbf{ImageNet models.} When training ResNet~\cite{he2016deep} models on ImageNet~\cite{deng2009imagenet}, we leverage the unmodified implementations from PyTorch~\cite{paszke2019pytorch} as well as their pre-trained floating-point checkpoints.
We use batch sizes of 64 images with an initial learning rate of 1e-4 that is reduced by a factor of 0.1 at epochs 30 and 50.
We fine-tune all models for 60 epochs using the standard stochastic gradient descent (SGD) optimizer with a weight decay of 1e-5.
Before fine-tuning, we apply the graph equalization and bias correction techniques proposed by~\citealt{nagel2019data} using a calibration set of 3000 images randomly sampled from the training dataset.
When applying our Euclidean projection-based weight initialization strategy discussed in Section~\ref{sec:l1_proj_init}, we do so after graph equalization, but before bias correction.
Finally, although common practice is to keep residuals as 32-bit additions~\cite{yao2021hawq}, we quantize our residual connections to 8 bits to reduce the cost of such high-precision additions.

\textbf{CIFAR10 models.} When training MobileNetV1~\cite{howard2017mobilenets} and ResNet18~\cite{he2016deep} to classify images on the CIFAR10 dataset~\cite{krizhevsky2009learning}, we follow the modified network architectures used by~\citealt{colbert2023a2q}.
These modifications reduce the degree of downsampling throughout these networks to yield intermediate representations that are more amenable to the smaller image sizes of CIFAR10.
For MobileNetV1, we use batch sizes of 64 images with an initial learning rate of 1e-3 that is reduced by a factor of 0.9 every epoch.
For ResNet18, we use batch sizes of 256 with an initial learning rate of 1e-3 that is reduced by a factor of 0.1 every 30 epochs.
We use a weight decay of 1e-5 for both models.
We initialize all quantized models from pre-trained floating-point checkpoints and fine-tune for 100 epochs using the standard SGD optimizer.
We again apply the graph equalization and bias correction techniques before fine-tuning, but using a calibration set of 1000 images.
We again use our Euclidean projection-based weight initialization strategy after graph equalization, but before bias correction.
Finally, we further quantize our residual additions to the same bit width specified for our hidden activations, \textit{i.e.}, $N$.

\textbf{BSD300 models.} When training ESPCN~\cite{shi2016real} and U-Net~\cite{ronneberger2015u} to upscale images by a factor of $3\times$ using the BSD300 dataset~\cite{martin2001database}, we again follow the modified architectures used by ~\citealt{colbert2023a2q}.
These modifications rely on the nearest neighbor resize convolution to upsample intermediate representations to improve model accuracy during training~\cite{odena2016deconvolution}, without impacting inference efficiency~\cite{colbert2021energy}.
For both models, we use batch sizes of 8 images with an initial learning rate of 1e-3 that is reduced by a factor of 0.999 every epoch and again use a weight decay of 1e-5.
We randomly initialize all models according to~\citealt{he2015delving} and train them from scratch for 300 epochs using the Adam optimizer~\cite{kingma2014adam}.
Similar to the CIFAR10 models, we quantize our residual additions in U-Net to the hidden activation bit width $N$.

\subsection{A2Q+ for Depthwise Separable Convolutions}
\label{appendix:dws_conv_ablation}

A2Q+ relies on zero-centering the weights for the purpose of alleviating the overly restrictive $\ell_1$-norm constraints of A2Q.
As discussed in Section~\ref{sec:discussion}, several studies have investigated the impact of zero-centering within the context of weight normalization~\cite{huang2017centered, qiao2019micro, li2019additive}.
While these works highlight the favorable properties of zero-centered weight normalization, such as stabilized pre-activation distributions and improved convergence, they seem to overlook an unfavorable property: implicit dimensionality reduction.

{Given $K$-dimensional weight vector $\vec{w}$, the zero-centering operation can be interpreted as a projection onto a $K-1$ hyperplane~\cite{yang2019mean}.
This implies that such a constraint reduces the degrees of freedom of the zero-centered weight vector, and such a reduction has a more significant impact with smaller $K$.
In the context of our work, we find that this introduces issues when handling layers with smaller dot product sizes, as is the case with the depthwise separable convolutions~\cite{sifre2014rigid} commonly used in MobileNets~\cite{howard2017mobilenets}.

Depthwise separable convolutions factorize the standard convolution into two chained operations: (1) a depthwise convolution that applies a single filter to each input channel; followed by (2) a pointwise convolution that applies a $1 \times 1$ kernel that combines the resulting output channels~\cite{sifre2014rigid, howard2017mobilenets}.
The size of pointwise convolution dot products is equivalent to the number of the input channels in the layer, which tend to be large.
However, size of depthwise convolution dot products is equivalent to the size of the kernels, which tend to be orders of magnitude smaller~\cite{howard2017mobilenets}.
Prior studies on zero-centered weight normalization  focus on benchmarks without these convolutions, \textit{e.g.}, VGGs and ResNets.
In the scope of our work, we find that zero-centering the weights of depthwise convolutions negatively impacts model accuracy.

\begin{figure}[t!]
\centering
\includegraphics[width=0.8\linewidth]{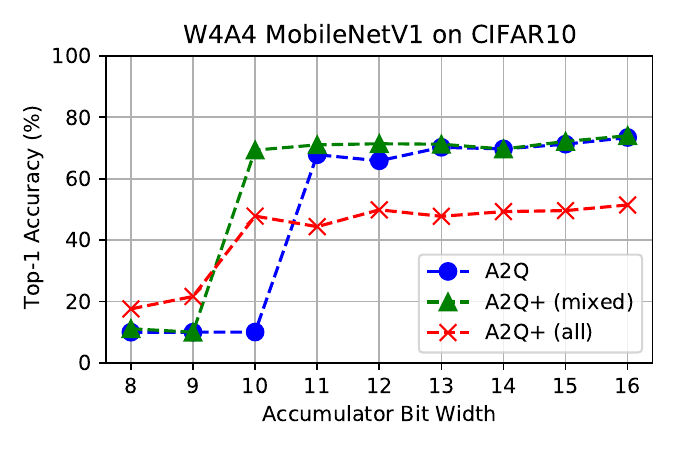}
\caption{We evaluate the impact of zero-centering on depthwise convolutions as we reduce the target accumulator bit width.
We visualize the maximum observed test top-1 accuracy when training a W4A4 MobileNetV1 model on CIFAR10.
We show that using A2Q for all depthwise convolutions and A2Q+ for all other hidden layers (\textcolor{ForestGreen}{\textbf{green triangles}}) outperforms uniformly applying A2Q (\textcolor{blue}{\textbf{blue circles}}) or A2Q+ (\textcolor{red}{\textbf{red crosses}}) to all hidden layers.}
\label{fig:ablation_dws_convs}
\end{figure}

\begin{figure*}[t!]
\centering
\includegraphics[width=0.9\linewidth]{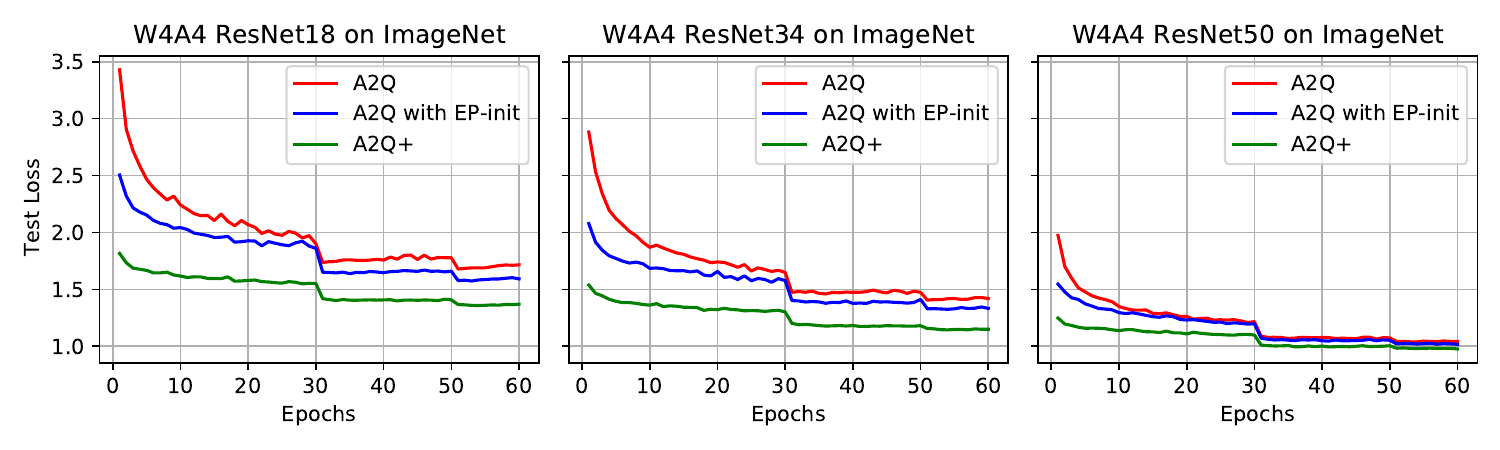}
\caption{We visualize the test cross entropy loss when training ResNet18, ResNet34, and ResNet50 to classify ImageNet images using 4-bit weights and activations (W4A4) and targeting 14-bit accumulation using A2Q.
We observe that our Euclidean projection initialization (EP-init) helps improve convergence.
Note that respective test top-1 accuracies are detailed in Table~\ref{tbl:imagenet}.}
\label{fig:imagenet_training_curves}
\end{figure*}

In Fig.~\ref{fig:ablation_dws_convs}, we evaluate the impact of zero-centering on depthwise separable convolutions.
We visualize the maximum test top-1 accuracy observed over 3 experiments when training MobileNetV1 to classify CIFAR10 images assuming 4-bit weights and activations (W4A4).
Using the hyperparameters detailed in Section~\ref{appendix:hyperparams}, we find that using A2Q for depthwise convolutions and A2Q+ for pointwise convolutions, we are able to recover model accuracy lost to implicit dimensionality reduction.
In addition, we are able to improve the trade-off between model accuracy and accumulator bit width.
Thus, we use the mixed depthwise separable convolution wherever possible in this work.

\subsection{Impact of Euclidean Projection Initialization}
\label{appendix:l1_proj_ablation}

In Section~\ref{sec:l1_proj_init}, we introduce a Euclidean projection-based weight initialization strategy designed to minimize the quantization error when initializing A2Q and A2Q+ models from pre-trained floating-point checkpoints.
We refer to this strategy as EP-init.
{Our results in Section~\ref{sec:more_models} show that EP-init significantly improves A2Q as a reference baseline for ImageNet models.}
This section provides a deeper empirical analysis of the initial weight quantization error.
We additionally discuss our ablation study that isolates the impact of EP-init across quantizers and target accumulator bit widths.

\subsubsection{Initial weight quantization error}
\label{appendix:init_quant_err_analysis}

A common practice when initializing QNNs from pre-trained floating-point checkpoints is to define scaling factor~$s$ to be the ratio given by Eq.~\ref{eq:scaling_init}~\cite{gholami2021survey, zhang2022learning, aggarwal2023post}.
Here, $\max(\vert \vec{w}_{\textrm{float}} \vert)$ is the maximum observed floating-point weight magnitude defined per-output channel and {$M$ is the target weight bit width defined per-tensor}.
\begin{equation}
s = \frac{\max (\vert \vec{w}_{\textrm{float}} \vert)}{2^{M - 1} - 1}
\label{eq:scaling_init}
\end{equation}
Using A2Q and A2Q+ to fine-tune QNNs from pre-trained floating-point checkpoints requires initializing two new learnable parameters: $g$ and $\vec{v}$.
One could trivially initialize $\vec{v}$ to be the pre-trained floating-point weight vector $\vec{w}_{\textrm{float}}$ and $g$ to be its $\ell_1$-norm such that $\vec{v} = \vec{w}_{\textrm{float}}$ and $g$ = $\Vert \vec{w}_{\textrm{float}} \Vert_1$ to ensure $\vec{w}$ = $\vec{w}_{\textrm{float}}$.
While this works well when targeting high-precision accumulators (\textit{e.g.}, 32 bits), we observe that A2Q-quantized networks are forced to quickly recover from extremely high losses when targeting low-precision accumulation scenarios
(\textit{e.g.}, 16 bits or fewer).

Figure~\ref{fig:imagenet_training_curves} visualizes the test cross entropy loss when training various W4A4 ResNets to classify ImageNet images while targeting 14-bit accumulators.
A2Q-quantized networks do not fully recover when na\"ively initialized.
Upon deeper investigation, we identify that this is in large part a consequence of A2Q clipping $g$ according to $T$ in Eq.~\ref{eq:a2q_v1_clip}.

\begin{figure}[t!]
\centering
\includegraphics[width=0.85\linewidth]{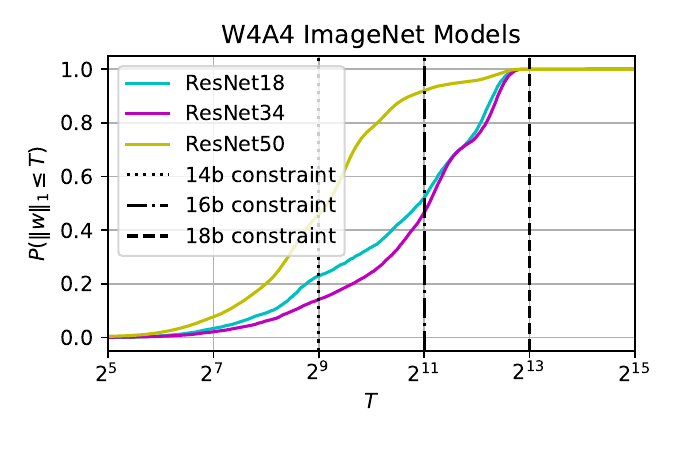}
\caption{We provide an empirical CDF to visualize the percentage of output channels in {various A2Q-quantized} W4A4 ResNets that inherently satisfies 14-, 16-, and 18-bit $\ell_1$-norm constraints when {first} initialized from a pre-trained ImageNet checkpoint.}
\label{fig:resnet18_cdf}
\end{figure}

\begin{figure*}[t!]
\centering
\includegraphics[width=0.9\linewidth]{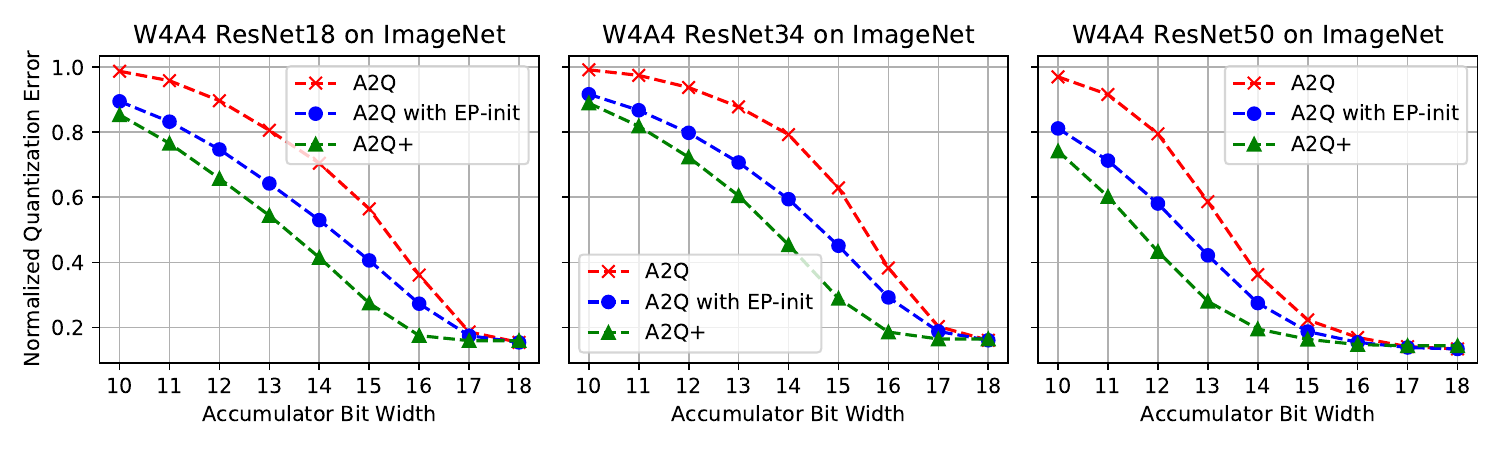}
\caption{We evaluate the normalized weight quantization error averaged over each output channel when initializing W4A4 ResNet variants from a pre-trained floating-point models trained on ImageNet.
As the accumulator bit width is reduced, we observe that our Euclidean projection initialization (EP-init) yields less error than na\"ive initialization for A2Q.
We additionally show that A2Q+ yields the lowest initial weight quantization error by combining EP-init with our new bound.}
\label{fig:weight_quant_error}
\end{figure*}

We first analyze the pre-trained floating-point ResNet checkpoints to demonstrate the breadth of this problem.
We evaluate $\Vert \vec{w}_{\textrm{float}} / s \Vert_1$ for each output channel in each hidden layer of each ImageNet model using the scaling factor definition provided in Eq.~\ref{eq:scaling_init}.
We visualize the results as an empirical cumulative distribution function (CDF) in Fig.~\ref{fig:resnet18_cdf}.
This CDF shows the percentage of channels in each W4A4 ImageNet model that inherently satisfies 14-, 16-, and 18-bit accumulator constraints assuming A2Q is the weight quantizer.
{While all per-channel weight vectors satisfy a 18-bit accumulator constraint upon initialization, we observe that only 52\% of ResNet18 channels inherently satisfy a 16-bit accumulator constraint and a mere 23\% inherently satisfy a 14-bit accumulator constraint.
For ResNet34, we observe 47\% satisfy the 16-bit constraint and 14\% satisfy the 14-bit constraint.
Interestingly, we observe that 92\% of ResNet50 channels inherently satisfy the 16-bit constraint and 46\% satisfy the 14-bit constraint.
We hypothesize this is because our accumulator constraints tighten with the width rather than the depth of a neural network.
This allows model capacity to increase without tightening of constraints.
It is important to note that these observations are dependent on the exact weight values of the pre-trained floating-point checkpoint.
We use the standard pre-trained floating-point checkpoints provided by PyTorch~\cite{paszke2019pytorch} within the scope of this work and leave an exhaustive analysis of other checkpoints for future work.

As a consequence of the $\ell_1$-norm constraints, na\"ively initializing $g$ such that $g = \Vert \vec{w}_{\textrm{float}} \Vert_1$ significantly increases the initial weight quantization error when $\Vert \vec{w}_{\textrm{float}} \Vert_1 > T$.
Consider again the same ResNet models and let the weight quantization error at initialization be defined as follows:
\begin{equation}
\frac{1}{2} \Vert Q(\vec{w}) - \vec{w}_{\textrm{float}} \Vert^2_2
\label{eq:weight_quantization_error}
\end{equation}
To demonstrate how this weight quantization error increases as the target accumulator bit width is reduced, we independently evaluate Eq.~\ref{eq:weight_quantization_error} for each output channel and plot the average in Fig.~\ref{fig:weight_quant_error}.
To account for the varied sizes of each layer in the network, we normalize the quantization error of each output channel by the squared $\ell_2$-norm of the floating-point weights, formally defined as $\frac{1}{2} \Vert \vec{w}_{\textrm{float}} \Vert^2_2$.
Noticeably, when initializing $g$ and $\vec{v}$, the average weight quantization error increases exponentially with the reduction in accumulator bit width regardless of the strategy.
In fact, when targeting 10-bit accumulation, $Q(\vec{w})$ is initialized with nearly 100\% sparsity across all output channels.
However, we are able to effectively minimize initial weight quantization error when using EP-init.
We additionally observe that combining this strategy with our new bound further reduces the initial weight quantization error as the $\ell_1$-norm constraints are relaxed.
We show this reduced weight quantization error yields improved model accuracy in Section~\ref{sec:more_models}.

\subsubsection{Ablation Study on CIFAR10}

When constructing our Pareto frontiers in Section~\ref{sec:per_acc_dt}, we applied EP-init to all A2Q and A2Q+ models to strictly compare the quantizers without the influence of initialization.
To isolate the influence of initialization, we detail an ablation study that further investigates the impact of EP-init on model accuracy.
Our analysis aims to further connect initial weight quantization error to model accuracy.
Thus, we focus on ResNet18 trained on the CIFAR10 dataset.
Building from the ImageNet analysis, we again focus on 4-bit weights and activations (W4A4).

We first analyze $\Vert \vec{w}_{\textrm{float}} / s \Vert_1$ for each output channel in each layer of the model using the scaling factor initialization defined in Eq.~\ref{eq:scaling_init}.
We again calculate the percentage of channels that inherently satisfy various $\ell_1$-norm constraints and visualize the analysis as an empirical CDF in Fig.~\ref{fig:ablation_l1_proj}a.
We observe that all channels natively satisfy an 18-bit accumulator constraint with only 53\% satisfying a 16-bit constraint and only 19\% satisfying a 14-bit constraint.

\begin{figure*}[t!]
\centering
\subfloat[]{\includegraphics[width=0.33\linewidth]{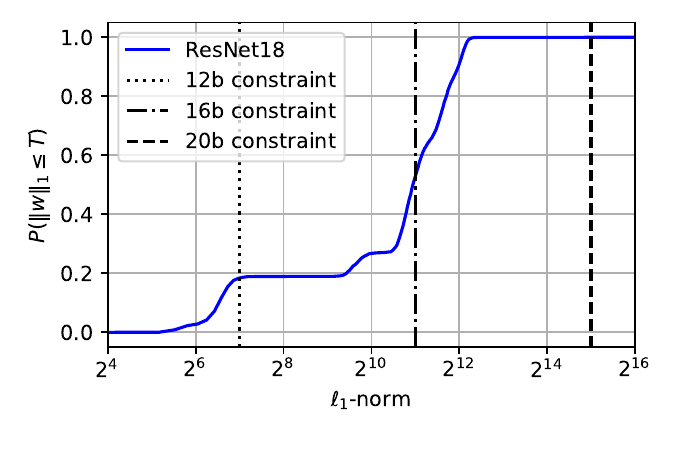}}
\subfloat[]{
\includegraphics[width=0.33\linewidth]{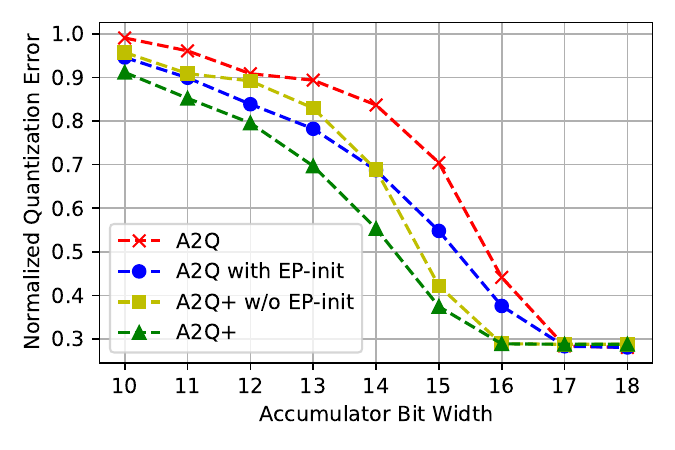}}
\subfloat[]{\includegraphics[width=0.33\linewidth]{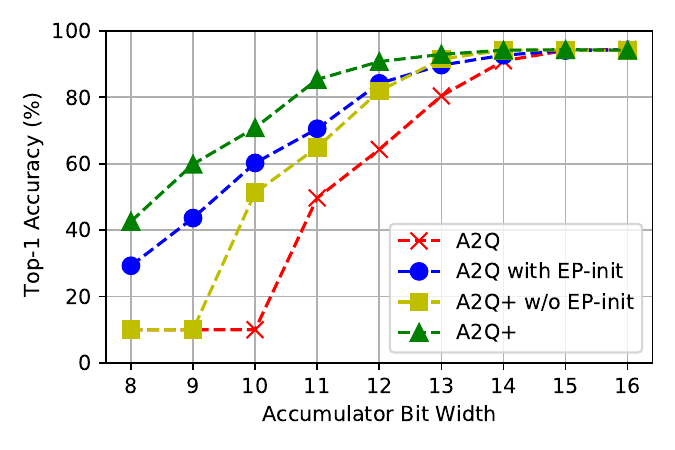}}
\caption{We evaluate the impact of Euclidean projection-based weight initialization (EP-init) as we reduce the target accumulator bit width for W4A4 ResNet18 trained on CIFAR10: (a) we visualize an empirical CDF to visualize the percentage of output channels that inherently satisfies various $\ell_1$-norm constraints; (b) we visualize the initial weight quantization error for both A2Q and A2Q+ with and without EP-init; and (c) we visualize the maximum observed test top-1 accuracy for both A2Q and A2Q+ both with and without EP-init.}
\label{fig:ablation_l1_proj}
\end{figure*}

Next, we evaluate the initial weight quantization error as the target accumulator bit width is reduced.
We again normalize the quantization error of each output channel by its squared $\ell_2$-norm.
We visualize the results in Fig.~\ref{fig:ablation_l1_proj}b, where we plot the average weight quantization error for each target accumulator bit width for both A2Q and A2Q+ models with and without EP-init.
Similar to our ImageNet results, we observe that combing our new bound with EP-init yields the lowest initial weight quantization error across target accumulator bit widths.

Finally, we evaluate the how initial weight quantization error translates to model accuracy as we reduce the target accumulator bit width for both A2Q and A2Q+ models.
In Fig.~\ref{fig:ablation_l1_proj}c, we visualize the maximum test top-1 accuracy observed over 3 experiments.
Intuitively, the strategy for initialization becomes more important as the target accumulator bit width is reduced.
The impact of EP-init increases as the expected initial weight quantization error increases, with the highest impact in the extremely low-precision accumulation regime.
In fact, we observe that EP-init yields up to a $+50\%$ increase in test top-1 accuracy for both A2Q+ and A2Q when targeting $9$- and $10$-bit accumulation, respectively.

\end{document}